\documentclass{article}

\usepackage{arxiv}

\usepackage{url}
\usepackage{graphicx}
\usepackage{subcaption}
\usepackage{multirow}
\usepackage{graphicx}
\usepackage{amsmath}
\usepackage{amssymb}
\usepackage{bm}
\usepackage{amsthm}
\usepackage{xr-hyper}
\usepackage{xr}
\usepackage{hyperref}

\usepackage{amsmath,amsfonts,bm}

\def\eqref#1{equation~\ref{#1}}

\def\1{\bm{1}}

\def\vk{{\bm{k}}}

\def\vp{{\bm{p}}}
\def\vq{{\bm{q}}}

\def\vv{{\bm{v}}}
\def\vw{{\bm{w}}}
\def\vx{{\bm{x}}}

\def\vz{{\bm{z}}}

\def\mK{{\bm{K}}}

\def\mQ{{\bm{Q}}}
\def\mR{{\bm{R}}}

\def\mT{{\bm{T}}}

\def\mV{{\bm{V}}}
\def\mW{{\bm{W}}}
\def\mX{{\bm{X}}}

\DeclareMathAlphabet{\mathsfit}{\encodingdefault}{\sfdefault}{m}{sl}
\SetMathAlphabet{\mathsfit}{bold}{\encodingdefault}{\sfdefault}{bx}{n}

\def\sB{{\mathbb{B}}}

\newcommand{\R}{\mathbb{R}}

\newtheorem{definition}{Definition}[section]
\newtheorem{theorem}{Theorem}[section]
\newtheorem{lemma}[theorem]{Lemma}
\newtheorem{problem}{Problem}

\title{Accelerating Material Property Prediction using Generically Complete Isometry Invariants}

\author{
 Jonathan Balasingham \\
  Department of Computer Science\\
University of Liverpool\\
Liverpool L69 3BX, UK\\
  \texttt{jbalasin@liverpool.ac.uk} \\
   \And
 Viktor Zamaraev \\
  Department of Computer Science\\
University of Liverpool\\
Liverpool L69 3BX, UK\\
  \texttt{Viktor.Zamaraev@liverpool.ac.uk} \\
  \And
 Vitaliy Kurlin \\
  Department of Computer Science\\
University of Liverpool\\
Liverpool L69 3BX, UK\\
  \texttt{Vitaliy.Kurlin@liverpool.ac.uk} 
}

\begin{document}
\maketitle
\begin{abstract}
Periodic material or crystal property prediction using machine learning has grown popular in recent years as it provides a computationally efficient replacement for classical simulation methods. A crucial first step for any of these algorithms is the representation used for a periodic crystal. While similar objects like molecules and proteins have a finite number of atoms and their representation can be built based upon a finite point cloud interpretation, periodic crystals are unbounded in size, making their representation more challenging. In the present work, we adapt the Pointwise Distance Distribution (PDD), a continuous and generically complete isometry invariant for periodic point sets, as a representation for our learning algorithm. The PDD distinguished all (more than 660 thousand) periodic crystals in the Cambridge Structural Database as purely periodic sets of points without atomic types. We develop a transformer model with a modified self-attention mechanism that combines PDD with compositional information via a spatial encoding method. This model is tested on the crystals of the Materials Project and Jarvis-DFT databases and shown to produce accuracy on par with state-of-the-art methods while being several times faster in both training and prediction time.
\end{abstract}


\section{Introduction}

A solid crystalline material is made up of a periodically repeated unit cell containing a motif of atoms (ions or molecules). Crystals can distinguish themselves by atomic types (chemical elements and possibly charges of ions) and by the geometry of atomic centers. Both of these aspects can determine the various properties of a crystal. Knowledge of these properties is pertinent for determining whether a crystal can be experimentally synthesized or is useful for a particular application.

Determination of property values can be done using ab initio calculations with techniques like density functional theory (DFT) \cite{dft1}. These techniques are often computationally expensive \cite{cohen2012challenges}. Further, they require extensive domain knowledge to be applied correctly, making them inaccessible. Recently, machine learning has become very popular as a substitute and has experienced success in decreasing computational costs while producing accurate predictions. 

\begin{figure}[h!]
    \centering
    \includegraphics[width=0.85\linewidth]{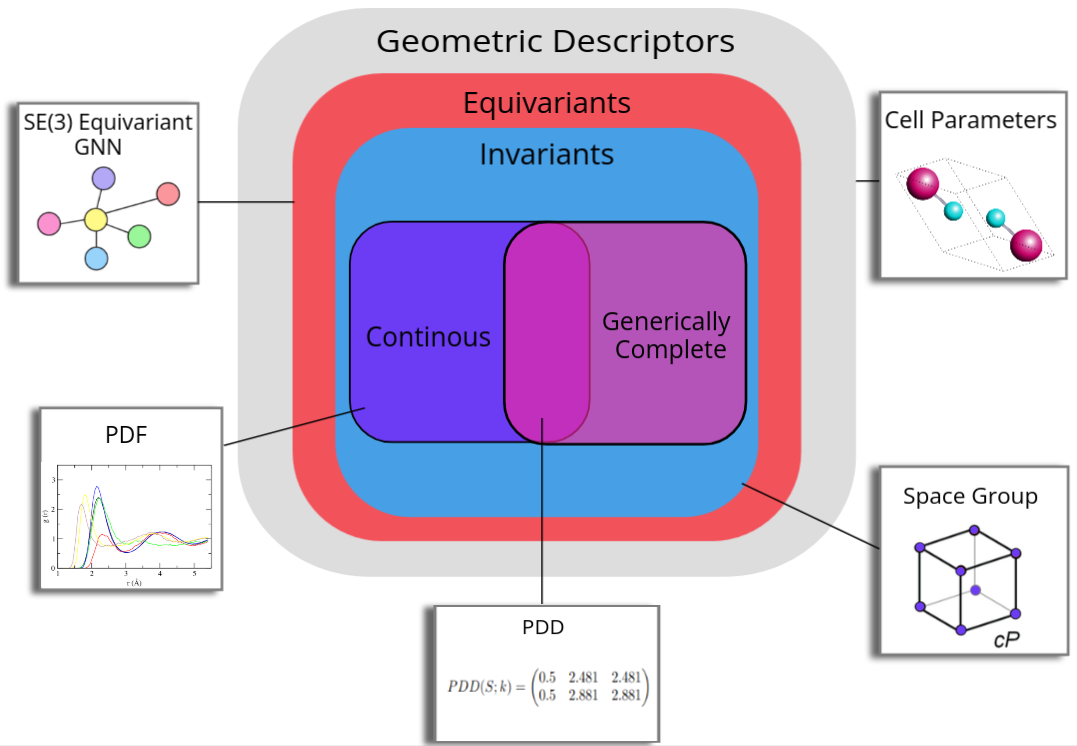}
    \caption{\small Classification of geometric descriptors for periodic crystals based on the properties possessed. Cell parameters consist of the unit cell lengths and angles, but this is ambiguous as there are an infinite number of unit cells. Space group is a label defined by the symmetry relations the crystal exhibits, but is sensitive to atomic perturbations. Equivariant GNNs cannot be used to distinguish periodic structures. PDF needs additional smoothing to retain continuity, introducing more parameters. The PDD is invariant, generically complete, and continuous under the EMD. }
    \label{fig:invariant_classification}
\end{figure}

Any learning algorithm requires an input representation that adequately describes the object of interest. Objects similar to crystals, like molecules, are often treated as finite point clouds. This makes their representation more easily constructible than a representation for crystals, which are not bounded in size.

While a crystal can be described in several ways, descriptors that are easily human-interpretable, such as unit cell parameters or atomic coordinates are not useful for machine learning algorithms. Atomic coordinates do not retain invariance under rigid motion. Unit cell based descriptors are also ambiguous as there are infinitely many valid unit cells for a single structure. Such ambiguities can result in different model outputs for the same structure. Techniques such as data augmentation \cite{quiroga2020revisiting_data_augmentation} and parameter sharing \cite{pmlr-v70-ravanbakhsh17a_parameter_sharing} can mitigate these effects but still do not guarantee the aforementioned consistency.

The \textit{structure-property relationship} \cite{structure_property} dictates that changes in the structure of a material result in changes in its properties. Distinction between crystals then allows for distinction between their respective property values. A machine learning algorithm (for a regression task) is a map from a crystal representation to value. If a representation cannot \textit{distinguish} periodic crystals then two different crystals can incorrectly be perceived to be the same and so will the output property values. Similarly, if the same crystal can be represented in different ways, consistent mapping cannot be guaranteed.

The fundamental model of a crystal is a periodic set of points at all atomic centers (even without atomic types), see details in Definition \ref{def:pps}.
Since crystal structures are determined in a rigid form, their strongest practical equivalence is \emph{rigid motion}, which is a composition of translations and rotations in $\mathbb{R}^n$.
We consider a slightly weaker \emph{isometry}, which is a composition of rigid motion and mirror reflections.
Two periodic point sets $S$ and $Q$ are \textit{isometric} if they are related by an isometry $f:\mathbb{R}^n\to\mathbb{R}^n$, so $f(S)=Q$. 
This work exclusively applies these concepts to dimension $n=3$. 

An isometry \emph{invariant} I is a property (descriptor or function) that is preserved under any isometry. The values of $I$ should be simpler than the initial periodic set, for example, a scalar, vector, or matrix. Not all invariants can be considered equally useful, however. Space groups, for example, reflect the symmetry of a given material but tiny perturbations of atomic coordinates can change the material's space group. Hence the important question is to quantify  \textit{how} different crystals are. Figure \ref{fig:invariant_classification} illustrates the relationship between geometric descriptors based on their properties. A practically useful invariant should satisfy the following conditions introduced by \cite{widdowson2022resolving}\label{sec:problem11}

\begin{problem}
    Find a function $I$ on all periodic point sets in $\mathbb{R}^n$ subject to the following conditions:
    
\begin{enumerate}
\setlength\itemsep{0.5em}
    \item[1a.] \textit{Invariance}: If two periodic point sets, $S$ and $Q$ are isometric, then $I(S) = I(Q)$. 
    \item[1b.] \textit{Generic Completeness}: If $I(S) = I(Q)$, then the two periodic sets are isometric.
    \item[1c.] \textit{Computability}: the invariant $I$, the metric $d$, and the reconstruction of $S$ can be computed in polynomial time with respect to the size of the motif of any periodic point set $S$.
    \item[1d.] \textit{Reconstructability}: any periodic set $S$ can be fully reconstructed from its invariant $I(S)$.
    \item[1e.] \textit{Metric}: There exists a distance function $d$ on the codomain $I$ that satisfies the following: 1) $d(I(S), I(Q)) = 0$ iff $I(S) = I(Q)$. 2) $d(I(S), I(Q)) = d(I(Q), I(S))$. 3) For any three periodic sets $S, Q,$ and $T$, $d(I(S), I(Q)) + d(I(Q), I(T)) \geq d(I(S), I(T))$
    \item[1f.] \textit{Lipschitz continuity}: If a periodic set $Q$ is obtained by shifting points within periodic set $S$ by at most $\epsilon$, then the distance between the two periodic sets can be bound according to some distance function $d$ such that $d(I(S), I(Q)) \leq C\epsilon$ for some fixed constant $C$.
\end{enumerate}
\end{problem}

In the present work, an isometry invariant called the Pointwise Distance Distribution (PDD), defined formally in Definition \ref{def:pdd}, which has properties $1a$ and $1c$ (see Theorem 5.1 of \cite{widdowson2022resolving}), and satisfies $1b$ and $1d$ for any periodic set in the general position \cite{widdowson2022resolving} along with sufficiently large $k$, and inclusion of the lattice (see theorem 4.4 of \cite{widdowson2022resolving}). Conditions $1e$ and $1f$ (see Theorem 4.3 of \cite{widdowson2022resolving}) are satisfied under the Earth Mover's Distance (EMD) between PDDs. 

The contribution of this work is a Transformer model \cite{attentionisallyouneed} which utilizes the PDD to make predictions on the properties of materials in a highly efficient manner compared to state-of-the-art models. In doing this, the gap between unambiguous crystal descriptors and machine learning models is bridged. Use of such a representation produces results on par or better than graph-based models, despite the additional structuring of data that comes with edges and edge embeddings. This model is faster in both prediction and training speed compared to two state-of-the-art models. To prove the method's robustness, the model is applied to the crystals of the Materials Project \cite{MP1} and Jarvis-DFT \cite{choudhary2020joint_jarvis_og}. Further experimentation on material classification, hyper-parameter sensitivity testing, and prediction of crystal properties without compositional information is included in Supplemental Materials Sections \ref{sec:addition_experiments} and \ref{sec:le_results}.

\subsection{Related Work}

Early works in crystal property prediction used more classical statistical methods like kernel regression \cite{LR1} before eventually moving towards deep learning \cite{deepANN}. More recent works have shifted to Graph Neural Networks (GNN)  \cite{CGCNN,ALIGNN,matformer,deeperGNN,iCGCNN,megnet,crysXPP,geoCGCNN,mtcgcnn,schutt2018schnet} due to their ability to make use of structured data. Several of these focus on predicting the properties of the crystals contained within the Materials Project \cite{MP1} using a multigraph representation where vertices represent atoms and edges are embedded with the pairwise distances to an atom's nearest neighbors. Some state-of-the-art models use line graphs to incorporate more geometric information like angles and dihedrals \cite{ALIGNN, ruff2023connectivity}. The derived line graphs can contain significantly more vertices and edges, incurring a higher computational cost. \cite{pmlr-v202-lin23m_potnet} take a physics principled approach and substitute the interatomic distances for interatomic potentials and capture a crystal's periodicity using the infinite sum of these potentials.

While effective in modeling crystal structures, graphs are discontinuous under perturbations \cite{perturbations}. Small movements in the atomic positioning can cause significant changes to the graph's topology. Some graphs are not unit cell invariant. Due to an infinite number of possible unit cells, the graph is then reliant on the data or the cell reduction technique used.


SE(3)-equivariant models such as Tensor Flow Networks \cite{equivariant}, SE(3)-Transformers \cite{equivariant_transformers}, and SE(3)-GNNs \cite{se3_equiv_gnn} impose constraints on the set of learnable functions of the network such that the output is equivariant with respect to the input points. While effective on finite point clouds, they offer no promise of completeness with periodic point sets which is necessary for distinguishing between structures. These architectures can also be beneficial when predicting properties that are equivariant with respect to rigid motion, but the crystal properties examined here are invariant to such symmetries, and thus invariance of the model output through either the input or model architecture is required.

In addition to the properties mentioned earlier, the invariant needs to be able to be adapted for a learning algorithm. Further, it needs a way to incorporate compositional information as invariants typically only consider structure. Some invariants have been adapted for use in machine learning algorithms such as symmetry functions \cite{symm1,symm2} and Voronoi cells \cite{voronoi}. Both of these, however, still lack continuity. The Partial Radial Distribution Function is invariant and continuous but is not complete for homometric crystals. Smooth Overlapped Atomic Positions \cite{soap} has been incorporated into models for property prediction, but is an invariant for atomic environments, not entire structures. Coulomb matrices use electrostatic interactions between atoms instead of Euclidean distances and are invariant and complete for molecules but have not been proven to retain these properties for the periodic case \cite{faber2015crystal}. Average Minimum Distance (AMD) \cite{widdowson2022average} is invariant and continuous and has been used to predict lattice energies via Gaussian Process Regression \cite{amd_le}, but is incomplete and does not currently have a way to incorporate compositional information. 

\section{Methods}

A periodic crystal can be represented as a periodic point set \cite{PPS} with points located at the atomic centers of the structure. They do not differentiate between the types of atoms and instead treat every point as unlabeled. A periodic point set (periodic set) can defined like so:

\begin{definition}[Periodic Point Set]\label{def:pps}
For a set of $n$ basis vectors $\vv_1 \ldots \vv_n \in \R^n$, the lattice $L$ is formed by the integer linear combinations of these basis vectors $\{ \sum_{i=1}^n c_i \vv_i \vert c_i \in \mathbb{Z}   \}$. The
unit cell is the parallelepiped $U = \{ \sum_{i=1}^n t_i \vv_i | t_i \in [0, 1) \}$. For a unit cell $U$, the motif $M$ is a finite subset of $U$. Then, a periodic point set $S$ of lattice $L$ and motif $M$ is defined by $\{ \bm{\lambda} + \vp : \bm{\lambda} \in L ,  \vp \in M \}$.
\end{definition}

The PDD of a periodic set is the $m \times (k+1)$ matrix where $m$ is the number of atoms in the motif $M$ and $k$ is a positive integer indicating the number of nearest neighbors to use. Each row corresponds to a point in the motif and the entries within the row consist of the Euclidean distance to each of this point's $k$-nearest neighbors within the entire periodic set $S$. The first entry of the row is assigned to be a weight equal to $\frac{1}{m}$ (the distances follow). Once the matrix is formed, rows that are the same are collapsed into a single row and their respective weights are added. Due to very small differences between rows caused by floating point arithmetic or atomic perturbations, it is common to use a tolerance, henceforth called the \textit{collapse tolerance}, that allows rows with small non-zero differences (e.g. with respect to $L_\infty$ distance) to be treated as the same. By collapsing rows in the PDD, the resulting matrix representation is always the same for a given crystal, regardless of the unit cell. Formally,

\begin{definition}[Pointwise Distance Distribution]\label{def:pdd}
    For a periodic set $S = L + M$ with a set of motif points $M = \{\vp_1, \ldots, \vp_m \}$ within a unit cell $U$ of lattice $L$, the uncollapsed PDD matrix for a parameter $k \in \mathbb{N}^+$ is a $m \times (k+1)$ matrix where the $i^{th}$ row consists of the row weight $w_i = \frac{1}{m}$ followed by the euclidean distances $d_1 \ldots d_k$ from the point $\vp_i$ to its $k$-nearest neighbors such that $d_1 \leq d_2 \ldots \leq d_k$. If a group of rows is found to be identical (or close enough using a valid distance measure within some tolerance) then the matrix rows are collapsed and the weights of the involved rows are summed. The resulting matrix will then have less than $m$ rows. 
\end{definition}

This matrix is referred to as $\textsc{PDD}(S;k)$ for a periodic set $S$ and positive integer $k$.

\subsection{Periodic Set Transformer}\label{sec:PST}

In our model, rather than being considered a matrix of values, the PDD will be considered a set of grouped atoms. A single group of atoms corresponds to the $k$-nearest neighbor distances in a given row within the PDD matrix. Each member of the set will carry the weight provided by the row in the PDD. Any set $A$ can trivially be turned into a weighted set by weighing each element by $\frac{1}{|A|}$. When the PDD is not collapsed, then there can be more than a single occurrence of any given element, making the uncollapsed PDD a multiset. Now, let $A$ be a multiset of the form $A = \{ a_{i}^{(j)} : i \in [1, \dots, n], j\in [1,\dots, n_i] \}$ where $n_i$ is the multiplicity of element $a_i$ and $a_{i}^{(j)}$ is the $j^{th}$ occurrence of element $a_i$. This multiset can be turned into a weighted set by assigning each element $a_i$ with the weight $\frac{n_i}{n}$. We can recover the influence of multiplicity by the use of weights in our model. 

\begin{figure}[h!]
    \centering
    \includegraphics[width=0.9\linewidth]{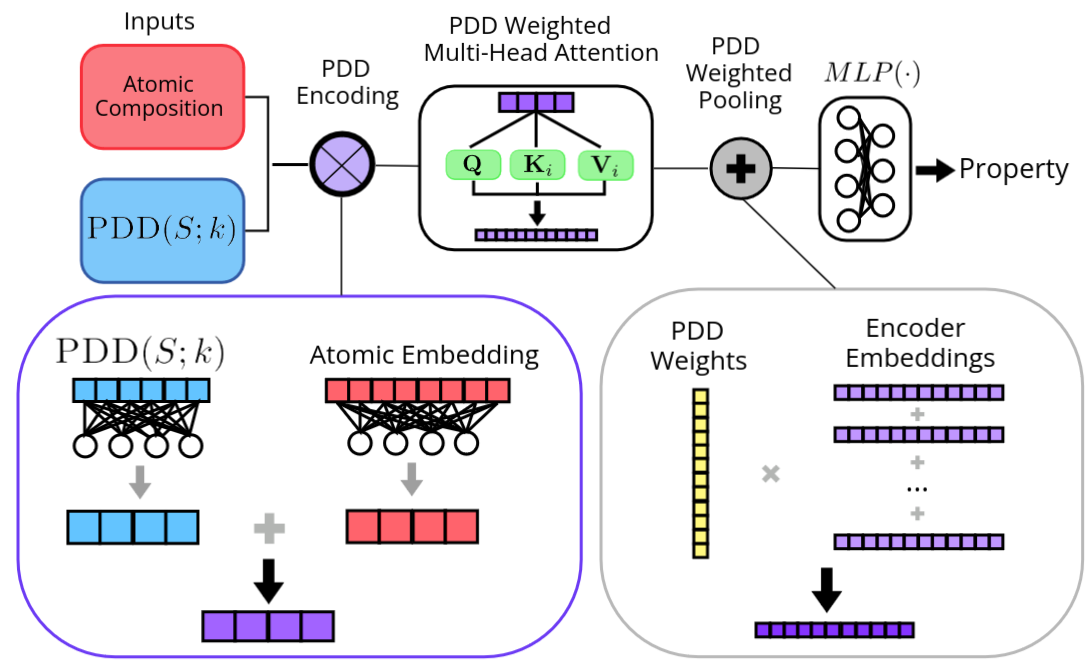}
    \caption{\small Overview of the architecture of the Periodic Set Transformer. PDD encoding is used to combine the structural information in the PDD with atomic
types. The weights of the PDD are incorporated in the attention mechanism and during the pooling of the embeddings to define the multiplicity of the input set.}
    \label{fig:pst}
\end{figure}

When a periodic crystal has its unit cell modified, the proportion of each atom is expanded or reduced. The use of weights captures this behavior in the form of a concentration or frequency. 

We use an attention mechanism to find the interactions between members of the set. The rows of the PDD contain the pairwise distance information, but they do not indicate which atoms these distances correspond to. Application of the attention mechanism can help the model learn these interactions. 

Let $\mR \in \R^{m \times k}$ be the PDD matrix containing $m$ rows without the associated weight column. Let $\vw \in \R^{m \times 1}$ be the column vector containing the weights from the PDD matrix. The initial embedding is $\mX^{(0)} = \mR \mW_d$ where $\mW_d \in \R^{k \times d}$ is the initial trainable weight matrix. The embedding is updated according to:
\begin{equation}
    \mX^{(1)} = \mX^{(0)} + SLP \Bigg( \sigma \left( \frac{\mQ \mK^T}{\sqrt{d}} \right) \mV \Bigg) 
\end{equation}
where $\mQ = \mX^{(0)}  \mW_Q$, $\mK = \mX_i^{(0)} \mW_K$ and $\mV = \mX^{(0)} \mW_{V}$ and $d$ is the embedding dimension of the weight matrices for the query, key, and value $\mW_Q, \mW_K,$ and $\mW_V$ respectively as described in \cite{attentionisallyouneed}. The function $\sigma$ is the softmax function with the PDD weights integrated into it; $\sigma$ is applied to each row $\vz$ of the input matrix, and $i$ and $j$ are used to index entries in $\vz$ and $\vw$. The $i^{th}$ entry of the output vector is defined by:

\begin{equation}\label{eq:attention_mechanism}
\sigma(\vz)_i = \frac{w_i e^{z_i}}{\sum_{j=1}^{m} w_j e^{z_j}}     
\end{equation}
The result is passed through a single-layer perceptron $SLP$. The layer normalization order described by \cite{prelayernorm} is used for increased stability during training. Equation (\ref{eq:attention_mechanism}) describes the case for a single attention head. When multiple attention heads are used, the PDD weights are applied to each individually and the result is concatenated before being passed to the SLP like so:

\begin{equation}
   \mX^{(1)} = \mX^{(0)} + SLP \Bigg( \bigoplus_{i=1}^h  \sigma \left( \frac{\mQ_i \mK_i^T}{\sqrt{d}} \right) \mV_i \Bigg) 
\end{equation}

where $h$ is the number of attention heads, $\oplus$ is the concatenation operator and $\mQ_i, \mK_i$ and $\mV_i$ are the query, key and value for the $i^{th}$ head.
This process is repeated $l$ times; this determines the depth of the model. The embeddings are finally pooled into a single vector by reincorporating the PDD weights into a weighted sum of the row vectors $\vx_i$ of the final embedding $\mX^{(l)}$. 

\begin{equation}\label{eq:pooling}
\vx = \sum_i w_i \vx_i    
\end{equation}

\noindent
This final embedding can be passed to a perceptron layer to predict the property value. 


This version of self-attention can be applied to a weighted set or distribution. The weights are applied in such a way that the output of the PST is invariant to an arbitrary splitting of rows within the PDD. We provide a formal proof of this in supplementary material Section \ref{sec:proof}. An overview of the Periodic Set Transformer (PST) architecture with PDD encoding (described in the next section) can be seen in Figure \ref{fig:pst}.

\subsection{PDD Encoding}

While structure is a powerful indicator of a crystal's properties, there may be datasets in which it is not the primary differentiator of a set of crystals. In such cases, the composition of the atoms contained within the material has a heavy influence. The previously described transformer does a good job of utilizing the structural information within the PDD but does not provide an obvious way to include atomic composition. 

Transformers for natural language processing tasks use positional encoding to allow the model to distinguish the position of words within a given sentence \cite{positional_encoding_nlp}. A recent transformer model, \textit{Uni-Mol} \cite{unimol}, which performed property prediction for molecules (among other tasks), used \textit{3D spatial encoding} first proposed by \cite{graphormer} to give the model an understanding of each atom's position in space, relative to one another. This encoding is done at the pair level, using the Euclidean distance between atoms and a pair-type aware Gaussian kernel \cite{3d_spatial}. A transformer model for finite $3D$ points clouds is provided by \cite{point_transformer} via \textit{vector} attention. The case for crystals is more difficult because they are not bounded in size and can exhibit many symmetries. Fortunately, by using the rows of the PDD we can distinguish each atom with structural information. We refer to this as \emph{PDD encoding}. 

When rows are grouped together, they are done so by having the same $k$-nearest neighbor distances. Though rare, it is possible for rows corresponding to different atom types to be collapsed. If this occurs, the selection of either atom type will result in information loss. To prevent this, we add the condition that the groups must be formed on the basis of having the same $k$-nearest neighbor distances and the same atomic species. In this case, the periodic point set has points that are labeled according to atomic type.  

For a periodic set $S$, let $PDD(S;k)$ be the resulting PDD matrix with parameter $k$. Let $\mR$ be $PDD(S;k)$ without the initial weight column and $\mT$ be the matrix whose rows correspond to the vector of atomic properties used to describe the type of atom associated with each row of $PDD(S;k)$. The initial set of embeddings for the attention mechanism is defined as $\mX^{(0)} = \mR \mW_s + \mT \mW_c$ where $\mW_s$ and $\mW_c$ are initial embedding weights. By starting with a linear embedding, the PDD row can be transformed to match the dimension of composition embedding. The parameter $k$ used can then be changed as needed to include distance information from further neighbors. 

\section{Results and Discussion}
\subsection{Prediction of Materials Project Properties}\label{sec:mp_results}

The model will be applied to the data within the Materials Project. To make fair comparisons to other models we report the performance according to \textit{Matbench} \cite{matbench}, which contains data for various crystal properties. The error rates are reported using five-fold cross-validation with standardized training and testing sets for each fold. Further, tuning is done according to the models' authors and thus our model can be compared to others more fairly.

The crystals in the Materials Project are highly diverse in composition. For all predictions, we include the composition of the crystal with PDD encoding. To incorporate this compositional information the \textit{mat2vec} atomic embeddings supplied by \cite{tshitoyan2019unsupervised_mat2vec} are used. The embeddings have empirically been found to produce better performance than the one-hot encoded method used by CGCNN \cite{crabnet}. They also have the added convenience of not missing any atomic property information for certain elements.

In Table \ref{tab:mpresults1} we report the average mean-absolute-error (MAE) across the five test sets. We include the reported accuracies of other models to allow for comparison. The selection of models aims to present a high diversity in approaches while also coming from relatively recent publications. 
\textit{CrabNet} \cite{crabnet} is the only other Transformer model listed on Matbench. This model, in terms of architecture, is the most similar to the PST. Additionally, the atomic embeddings used to describe each chemical element are the same as those used in our model. The majority of models used in crystal property prediction use GNNs. Matbench features several of these, but the model that provides the best results on several properties is \textit{coGN}. \textit{coGN} is a GNN that includes angular and dihedral information through the use of line graphs. As such, the amount of information used is significantly more than the PST, which uses the distribution of pairwise distances. \textit{Crystal Twins} (CT) \cite{crystaltwins} is a model based on the convolutional layer developed by CGCNN \cite{CGCNN} (also used by several other models \cite{iCGCNN,moformer,crysXPP}) that uses self-supervised learning to create embeddings based on maximizing the similarity between augmented instances of a crystal.

\begin{table}[h]
\small
\caption{\label{tab:mpresults1}%
\small Five-fold cross-validation prediction MAE and standard deviation of MAE for properties of the crystals in the Materials Project. Bold values indicate the best (lowest) error rate while underlined values indicate the second-best error rate. PST performance is reported using PDD Encoding with a tolerance of $10^{-4}$ and $k=15$.}
\centering
\begin{tabular}{|p{2.4cm}|p{1.5cm}|c|c|c|c|c|}
\hline 
Property&
Units&
\hfil PST&
CrabNet&
\hfil coGN &
\hfil CrystalTwins\\
\hline
Formation Energy & $eV/atom$ & \hfil \underline{0.032 $\pm$ 0.0003} & \hfil 0.086 $\pm$ 0.001 & \hfil \textbf{0.021 $\pm$ 0.0003}  & \hfil 0.037 $\pm$ 0.001\\
Band Gap Energy & $eV$ & \hfil \underline{ 0.210 $\pm$ 0.002} & \hfil 0.266 $\pm$ 0.003 & \hfil \textbf{0.156} $\pm$ \textbf{0.002} & \hfil 0.264 $\pm$ 0.011 \\
Shear Modulus & $log_{10}(GPa)$ & \hfil \underline{0.074 $\pm$ 0.001} & \hfil 0.101 $\pm$ 0.002 & \hfil \textbf{0.069 $\pm$ 0.001} & \hfil 0.086 $\pm$ 0.004 \\
Bulk Modulus & $log_{10}(GPa)$ & \hfil \underline{0.056 $\pm$ 0.003} &  \hfil 0.076 $\pm$ 0.003 & \hfil \textbf{0.053 $\pm$ 0.003} & \hfil 0.067 $\pm$ 0.003 \\
Refractive Index & n/a & \hfil \textbf{0.290 $\pm$ 0.078} & \hfil 0.323 $\pm$ 0.071 & \hfil \underline{0.309 $\pm$ 0.086} & \hfil 0.417 $\pm$ 0.080 \\
Phonon Peak & $1/cm$ & \hfil \textbf{29.40 $\pm$ 1.40 } & \hfil 57.76 $\pm$ 5.73 &\hfil \underline{29.71 $\pm$ 1.99} & \hfil 48.86 $\pm$ 7.69 \\
Exfoliation Energy & $meV/atom$ & \hfil \textbf{31.15 $\pm$ 9.566} & \hfil 45.61 $\pm$ 12.24 & \hfil \underline{37.16 $\pm$ 13.68} & \hfil 46.79 $\pm$ 19.92  \\
Perovskites FE & $eV/cell$ & \hfil \underline{0.030 $\pm$ 0.001} & \hfil 0.406 $\pm$ 0.007 & \hfil \textbf{0.027 $\pm$ 0.001 } & \hfil 0.042 $\pm$  0.001 \\
\hline
\end{tabular}
\end{table}

The PST and CrabNet share similarities in their construction. Both use \textit{mat2vec} atomic embeddings and utilize a Transformer architecture with self-attention. While CrabNet uses fractional encoding to embed the multiplicity of each element type, we opt for the PDD-weighted attention mechanism and pooling described by Equations (\ref{eq:attention_mechanism}) and (\ref{eq:pooling}). The PST also uses PDD encoding to add structural information. This could also be done for CrabNet, but the combination of both fractional and PDD encoding is not guaranteed to aid in performance and simple summation of the encodings can cause ambiguities in the final embeddings, reducing performance. Across all properties the PST outperforms CrabNet, further indicating the usefulness of PDD encoding.

The performance disparity between the PST and coGN on formation and band gap energy can be difficult to discern. GNNs allow embeddings at the vertex and edge level. These embeddings can not only carry different information but allow for simultaneous updates to each of these embeddings, adding richness to the learned representation. CoGN takes this further and updates the original edges with message-passing from the derived line graph, allowing the inclusion of angular information. The edges of the line graph are further updated by its line graph, incorporating dihedral angles. These updates allow the model to learn a better representation of a crystal in latent space, which can be necessary for these larger datasets. This points to a current limitation for the PST. By using PDD encoding, we effectively limit opportunities for such updates to a single embedding representing both the atom's properties and its structural behavior.

In Table \ref{tab:matbench_times} the time taken to train and test the PST and coGN are listed. Our model was trained for only 250 epochs while coGN is trained for 800 epochs. While this accounts for a large portion of the disparity, the time taken to train each model per epoch is also faster for our model in all properties except the refractive index. A batch size of 32 is used for datasets containing less than $5,000$ samples. For exfoliation energy and phonon peak, PST takes $68.1\%$ and $92.6\%$ of the time of coGN per epoch. The larger batch size of coGN allows it to have greater GPU utilization and thus, better training efficiency. For properties with greater than $5,000$ samples, the batch size for both models is the same. For these properties, the training time per epoch is between $65-70\%$ of coGN. 

The prediction time disparity is more pronounced. For band gap and formation energy, the PST makes predictions approximately five times faster than coGN. Using nested line graphs introduces significant computational cost, but coGN is able to shrink the size of the graph by using the atoms in the asymmetric unit. The unit cell based approach was initially proposed by CGCNN \cite{CGCNN} and used in several follow-up works \cite{crysXPP,iCGCNN,ALIGNN,crystaltwins}, but the unit cell is inherently ambiguous and unnecessarily large in terms of the number of atoms needed to fully describe a crystal's symmetry. The PDD at a collapse tolerance of exactly zero will have a number of rows less than or equal to the number of atoms in the asymmetric unit. At higher tolerances, this number will be reduced until reaching the number of unique chemical elements in the crystal. 

\begin{table}[h]
    \centering
    \caption{\label{tab:matbench_times}%
\small Single-fold prediction (measured in seconds) and training time (measured in minutes) for the PST (using $k=15$ and a collapse tolerance of $10^{-4}$) and coGN \cite{ruff2023connectivity} on Matbench \cite{matbench}. Training and prediction was done using an Nvidia RTX 3090. Time does not include evaluation of the models on the validation sets or data pre-processing times. }
    \begin{small}
  \begin{tabular}{|p{3.75cm}|c|c|c|c|c|}
    \hline
    \multirow{2}{*}{\hfil Property} &
        \multirow{2}{*}{\hfil Samples} &
      \multicolumn{2}{c|}{Training Time (min.) } &
      \multicolumn{2}{c|}{ Prediction Time (sec.) }\\
      \cline{3-6}
   &  &  PST &  coGN &  PST &  coGN    \\
    \hline
Formation Energy   & 132,752 & 159.1 & 772.1 & 2.79 & 15.25 \\ %
Band Gap           & 106,113 & 126.6 & 602.1 & 2.38 & 11.88 \\ %
Perovskites FE     & 18,928  & 13.41 & 62.37 & 0.31  & 2.93 \\ %
Bulk Modulus       & 10,987  & 8.47  & 42.24 & 0.23 & 1.99 \\ %
Shear Modulus      & 10,987  & 8.38  & 41.23 & 0.22 & 2.05 \\ %
Refractive Index   & 4,764   & 6.89  & 20.23 & 0.12  & 1.29 \\ %
Phonon Peak        & 1,265   & 1.81  & 6.25  & 0.04  & 0.87 \\ %
Exfoliation Energy & 636     & 0.89  & 4.18  & 0.02  & 0.79 \\ %
    \hline
  \end{tabular}
    \end{small} 
\end{table}

\subsubsection{Ablation Study}

In Table \ref{tab:ablation_by_input_component} we list the results for each "component" within the Periodic Set Transformer. In the row indicating PDD as the component, we train and test the model using only the structural information within the PDD. In the "Composition" component we pass the atomic encoding for the elements in the crystal and their concentration in the form of the PDD weights to the model without the PDD encoding. The model is run using $k = 15$ and a collapse tolerance of exactly zero. 

By separating out each component of the model, we can interpret the importance of each to a particular property. Properties that experience a more significant decrease in performance when the PDD encoding is not used, can be ascribed to be more dependent on structural information. In all cases, the combination of both the composition and PDD encoding results in significantly lower error rates. We can conclude that this encoding method is effective in combining the structural and compositional information of a crystal structure.

\begin{table}[h]
    \centering
        \caption{ \small Effect of PDD encoding on prediction MAE of the Materials Project crystals. Results are separated by input components where "Composition" uses only the \textit{mat2vec} atomic embeddings \cite{tshitoyan2019unsupervised_mat2vec} and "PDD" uses only the PDD. Errors in bold indicate the best performance and underlined errors indicate the second-best performance  (lower is better $\downarrow$). }
    \begin{small}
  \begin{tabular}{|p{4cm}|p{2cm}|p{2cm}|p{2cm}|}
    \hline
    \multirow{2}{*}{\hfil Property (units)} &
      \multicolumn{3}{c|}{Component MAE $\downarrow$ } \\
      \cline{2-4}
    & \hfil Composition & \hfil PDD  & \hfil PST    \\
    \hline
    Band Gap \hspace{0.1cm} \small{$eV$} & \hfil \underline{0.273} & \hfil 0.596 & \hfil \textbf{0.212}  \\ 
    Formation \hspace{0.1cm} \small{$eV/atom$} & \hfil \underline{0.088} & \hfil 0.421 & \hfil \textbf{0.032}  \\ 
    Shear Modulus \hspace{0.1cm} \small{$log_{10}(GPa)$} & \hfil \underline{0.107} & \hfil 0.132 & \hfil \textbf{0.075} \\ 
    Bulk Modulus \hspace{0.1cm} \small{$log_{10}(GPa)$} & \hfil \underline{0.080} & \hfil 0.115 & \hfil \textbf{0.055} \\ 
    Refractive Index  & \hfil \underline{0.352} & \hfil 0.451 & \hfil \textbf{0.292} \\ 
    Phonon Peak \hspace{0.1cm} \small{$1 / cm$} & \hfil \underline{50.39} & \hfil 74.71 & \hfil \textbf{27.75} \\ 
    Exfoliation \hspace{0.1cm} \small{$meV/atom$} & \hfil 46.91 & \hfil \underline{39.35} & \hfil \textbf{31.55} \\ 
    Perovskites FE \hspace{0.1cm} \small{$eV/cell$} & \hfil 0.621 & \hfil \underline{0.393} & \hfil \textbf{0.030} \\ 
    \hline
  \end{tabular}
    \end{small} 
    \label{tab:ablation_by_input_component}
\end{table}

In Table \ref{tab:ablation_by_pdd_component} the effect of including the PDD weight in the attention mechanism described in Equation (\ref{eq:attention_mechanism}) and in the pooling layer described by Equation (\ref{eq:pooling}) is listed. A collapse tolerance of zero is used to remove any regularization effect (further described in Section \ref{sec:mp_results_by_tol}). 

The exclusion of weights from both the attention mechanism and pooling decreases accuracy significantly. Doing this removes all indications of multiplicity making discernment of crystals more difficult. The inclusion of the weights in the pooling layer is more impactful than when applied in the attention mechanism. The use of the weights in the pooling layer alone allows the model to perform better when the number of samples in the dataset is low. Datasets with fewer samples likely have less diversity amongst their crystals, making the need for recognizing the multiplicity of atoms less necessary.

\begin{table}[h]
    \centering
        \caption{ \small Effect of including the PDD weights as defined by Equations (\ref{eq:attention_mechanism}) and (\ref{eq:pooling}) on prediction MAE of the Materials Project crystals. Results for "No weights" use mean pooling and a normal softmax function. Errors in bold indicate the best performance and underlined errors indicate the second-best performance (lower is better $\downarrow$). }
    \begin{small}
  \begin{tabular}{|p{3.75cm}|p{2.1cm}|p{1.9cm}|p{1.7cm}|p{1.6cm}|}
    \hline
    \multirow{2}{*}{\hfil Property (units)} &
      \multicolumn{4}{c|}{PDD Weight Inclusion MAE $\downarrow$ } \\
      \cline{2-5}
    & \hfil No Weights & \hfil Attention Only & \hfil Pooling Only & \hfil PST    \\
    \hline
    Band Gap \hspace{0.1cm} \small{$eV$}               & \hfil 0.278 & \hfil 0.244 & \hfil \underline{0.219} & \hfil \textbf{0.212}  \\
    Formation \hspace{0.1cm} \small{$eV/atom$}         & \hfil 0.045 & \hfil 0.037 & \hfil \underline{0.035} & \hfil \textbf{0.032}  \\
    Shear Modulus \hspace{0.1cm} \small{$log_{10}(GPa)$} & \hfil 0.080 & \hfil 0.077 & \hfil \underline{0.076} &  \hfil \textbf{0.075} \\
    Bulk Modulus \hspace{0.1cm} \small{$log_{10}(GPa)$}  & \hfil 0.059 & \hfil 0.059 & \hfil \underline{0.056} &  \hfil \textbf{0.055} \\
    Refractive Index                                     & \hfil 0.314 & \hfil \textbf{0.284} & \hfil \underline{0.288} & \hfil 0.292 \\
    Phonon Peak \hspace{0.1cm} \small{$1 / cm$}          & \hfil 31.02 & \hfil 28.84 & \hfil \underline{27.96} &  \hfil \textbf{27.75} \\
    Exfoliation \hspace{0.1cm} \small{$meV/atom$}        & \hfil 35.59 & \hfil 32.52 & \hfil \underline{31.83} & \hfil \textbf{31.55} \\
    Perovskites FE \hspace{0.1cm} \small{$eV/cell$}       & \hfil 0.031 & \hfil \underline{0.030} & \hfil 0.031 & \hfil \textbf{0.030}  \\
    \hline
  \end{tabular}
    \end{small} 

    \label{tab:ablation_by_pdd_component}
\end{table}

\subsection{Prediction of Jarvis-DFT Properties}

The \textit{Jarvis-DFT} dataset \cite{choudhary2020joint_jarvis_og} is a commonly used set of materials with VASP \cite{kresse1996efficiency_vasp} calculated properties. The list of properties computed for the materials within the dataset is more extensive than that of the Materials Project. Its inclusion provides further evidence of the robustness of the model on an even wider variety of crystal properties. 



The prediction MAE produced by PST and \textit{Matformer} for $12$ different properties from the dataset are included in Table \ref{tab:jarvis_results}. For Matformer, we retrain the model to ensure the training and testing sets are the same. We use the default parameters for the model defined by the authors' codebase. We make one alteration to the training procedure; the number of epochs trained is reduced to 250. The number of epochs is the same as for our model. 

\begin{table}[h]
\small
\caption{\label{tab:jarvis_results}%
\small Prediction MAE on the properties of the Jarvis-DFT dataset using the PST and Matformer. Results for Matformer \cite{matformer} are included for comparison. PST uses PDD encoding with $k=15$ and a collapse tolerance of $10^{-4}$. Bolded values indicate the best performance. The Mean-Absolute-Deviation (MAD) of the test set is included. }
\centering
\begin{tabular}{|p{3.9cm}|p{1.7cm}|c|c|c|c|c|}
\hline 
Property&
Units&
Samples&
Test MAD&
PST&
Matformer
\\
\hline
Formation Energy & $eV/atom$ & 55,723 & 0.87 & 0.047 & \textbf{0.033}\\ %
Band Gap (OPT) & $eV$ & 55,723 & 0.99 & 0.172 & \textbf{0.150} \\
Total Energy & $eV/atom$ & 55,723 & 1.78 & 0.051 & \textbf{0.036} \\
Ehull & $eV$ & 55,371 & 1.14 & \textbf{0.052} & 0.072\\
Bulk Modulus & $GPa$ & 19,680 & 52.80 & \textbf{10.76} & 11.70 \\ 
Shear Modulus & $GPa$ & 19,680 & 27.16 & \textbf{9.523} & 10.13 \\ 
Band Gap (MBJ) & $eV$ & 18,172 & 1.79 & \textbf{0.289} & 0.304 \\ 
Spillage & - & 11,377 &  0.52  & \textbf{0.367} & 0.373\\
SLME (\%) & - & 9,068 &  10.93  & \textbf{4.61} & 4.712\\ 
Max. piezo. stress coeff $(e_{ij})$ & $Cm^{-2}$ & 4,799 & 0.26 & \textbf{0.127} & 0.243 \\ 
Max. piezo. strain coeff $(d_{ij})$ & $CN^{-1}$ & 3,347 & 24.57 & \textbf{13.09} & 18.03 \\ 
Exfoliation Energy & $meV/atom$  & 813 & 62.63 & \textbf{30.91} & 55.04\\

\hline
\end{tabular}
\end{table}

The PST outperforms Matformer in nine of the twelve properties tested. In particular, properties for which data is sparse yield results that favor the PST significantly (i.e. exfoliation energy, $e_{ij}$ and $d_{ij}$). Jarvis-DFT has two band gap values that are computed for its crystals, one which uses the optimized Becke88 functional (OPT) \cite{klimevs2009chemical_opt_bg} and the other uses the Tran-Blaha modified Becke Johnson potential (MBJ)\cite{tran2009accurate_mbj_functionals}. The latter is more accurate (when compared to experimentally observed values) but also more computationally expensive. For this reason, there are significantly fewer computed values in the database. Interestingly, the PST produces a smaller error for the more accurate band gap values compared to Matformer, but a larger error for the less accurate OPT calculated values. A possible reason for this is the smaller sample size for which the PST has shown to be more effective. The disparity in performance for formation and total energy can be attributed to Matformer's architecture which uses a GNN that updated both node and edge embeddings. This additional level of expression is helpful particularly when the size of the data grows larger, though it does come with added computational cost. 

Matformer has been shown to produce even better results than the previous state-of-the-art model \textit{ALIGNN} \cite{ALIGNN} while taking roughly a third of the time to do both training and prediction.  In Table \ref{tab:jarvis_times}, the training and prediction time for each of the properties in the Jarvis-DFT dataset is reported for the PST and Matformer. For the training time, the validation and pre-processing times are not included. The prediction time listed is the number of seconds taken to make predictions on the test set.

\begin{table}[h]
    \centering
    \caption{\label{tab:jarvis_times}%
\small Prediction (measured in seconds) and training time (measured in minutes) for the PST and Matformer \cite{matformer} on Jarvis-DFT datasets. Training and prediction was done using an Nvidia RTX 3090. Time does not include evaluation of the models on the validation sets or data pre-processing times. }
    \begin{small}
  \begin{tabular}{|p{3.75cm}|c|c|c|c|c|}
    \hline
    \multirow{2}{*}{\hfil Property} &
        \multirow{2}{*}{\hfil Samples} &
      \multicolumn{2}{c|}{Training Time (min.) } &
      \multicolumn{2}{c|}{ Prediction Time (sec.) }\\
      \cline{3-6}
   &  &  PST &  Matformer &  PST &  Matformer    \\
    \hline
Formation Energy  & 55,723  & 41.36 & 345.8 & 0.329 & 29.77\\ 
Band Gap (OPT)    & 55,723  & 41.62 & 343.9 & 0.347 & 29.86\\ %
Total Energy      & 55,723  & 41.65 & 349.1 & 0.349 & 29.79\\ %
Ehull             & 55,371  & 40.69 & 348.9 & 0.352 & 28.93\\ 
Bulk Modulus      & 19,680  & 14.12 & 93.33 & 0.135 & 11.12 \\ 
Shear Modulus     & 19,680  & 14.45 & 93.70 & 0.123 & 10.69 \\ 
Band Gap (MBJ)    & 18,172  & 13.38 & 118.7 & 0.107 & 9.71 \\ 

Spillage          & 11,377  & 5.74 & 70.8  & 0.066  & 6.01\\ 
SLME (\%)         & 9,068   & 4.62 & 58.75 & 0.055  & 4.82 \\ 

Max. piezo. stress coeff $(e_{ij})$ & 4,799 & 3.52 & 23.15 & 0.029 & 2.57 \\ 
Max. piezo. strain coeff $(d_{ij})$ & 3,347 & 2.44 & 15.38 & 0.026 & 1.79 \\ 

Exfoliation Energy                  & 813   & 0.63 & 5.30   & 0.008 & 0.41\\ 

    \hline
  \end{tabular}
    \end{small} 
\end{table}

In the closest training time comparison, the PST is still more than six times faster than Matformer. The training times for all properties fall between six and twelve times faster for the PST compared to Matformer. The performance increase can be attributed to several factors. Primarily, Matformer relies on a line graph (similar to coGN \cite{ruff2023connectivity}) in order to update edge embeddings. While this increases the information used and leads to richer learned embeddings, the size of line graphs is considerably larger than the graph they are derived from. This, in turn, incurs a higher computational cost. 

The difference in prediction times is more significant. Exfoliation energy is predicted over fifty times faster using the PST than with Matformer. This is the closest the two models perform to each other. Notably, exfoliation energy also has the fewest samples. For the bulk of the other properties, the speedup ranges between eighty and ninety times faster for the PST.

\section{Conclusion}
The PDD is a representation for periodic crystals which is invariant to rigid motion and independent of unit cell. By using weights and creating a distribution, the PDD is able to represent an infinitely spanning object by its finite forms of behavior. Further, by collapsing rows in the PDD, the resulting representation can also be much smaller in comparison to the number of atoms within the unit cell, even when the cell is reduced. 

The model is applied to the crystals of the Materials Project and Jarvis-DFT on a variety of material properties. Despite using less information in the model than more commonly employed graph-based models, the PST is able to produce results on par or even exceeding that of models like coGN and Matformer while taking significantly less time to train and make predictions.

\bibliographystyle{unsrt}  
\bibliography{references}  

\appendix
\newpage

\begin{center}
\textbf{\Large{Accelerating Material Property Prediction using Generically Complete Isometry Invariants: Supplemental Material}}

\vspace{0.5cm}
Jonathan Balasingham, Viktor Zamaraev, and Vitaliy Kurlin
\end{center}

\vspace{1cm}

\section{Equivalence of Distributions in the PST}\label{sec:proof}

If a PDD is arbitrarily expanded or collapsed, the PST should produce the same results as these PDDs are considered equivalent. The same can be said of any input distribution. This is proven here:

\begin{lemma}
Let $A$ and $B$ be weighted multisets each containing elements from the set $S = \{\vx_1,\ldots \vx_n\}$. Each element $\vx_i \in \R^{1 \times n}$ occurs with multiplicity $m_i^{(a)} \in \mathbb{N}^{+}$ and $m_i^{(b)} \in \mathbb{N}^{+}$ in $A$ and $B$ respectively. Each element also carries weight $w_i^{(a)} \in \R^{+}$ and $w_i^{(b)} \in \R^{+}$ for $A$ and $B$ respectively. The application of the Periodic Set Transformer will yield equivalent output if

\begin{equation}\label{eq:condition}
    w_i^{(a)} m_i^{(a)}  =   w_i^{(b)} m_i^{(b)}, \hspace{0.5cm} \forall i \in \{1,\ldots,n\}
\end{equation}

\end{lemma}

\begin{proof}
To prove that the output of the PST is equivalent for $A$ and $B$, it is sufficient to prove that the output of $\sigma$ (defined in Equation (\ref{eq:attention_mechanism})) and the pooling layer (defined in Equation (\ref{eq:pooling})) are the same for $A$ and $B$.
Let $\vq_i = \vx_i\mW_Q$, $\vk_i = \vx_i\mW_K$, and $\vv_i = \vx_i\mW_V$ be the query, key, and value vectors produced by the weight matrices $\mW_Q, \mW_K, \mW_V \in \R^{n \times d}$. First, note the pre-softmax attention weight from $\vx_i$ to $\vx_j$ can expressed as:
\begin{align}
    a_{ij} &= \frac{\vq_i \vk_j^{T}}{\sqrt{d}} 
\end{align}
and is independent of both weight and multiplicity. Further, notice that the summation of an expression over $A$ or $B$ can be rewritten using its multiplicities. For $A$ this is done like so:

$$ \sum_{t=1}^{|A|} (\cdot) = \sum_{s=1}^{n} m_s^{(a)} (\cdot)  $$
Using this equivalence, the function $\sigma$, used to calculate the attention weight from $\vx_i$ to $\vx_j$ for $A$ and $\sB$, is written as:
\begin{align}\label{eq:post_softmax_attention_A}
    \alpha_{ij}^{(a)} &= \frac{w_i^{(a)} exp(a_{ij})}{\sum_{k=1}^{n} w_k^{(a)} m_k^{(a)} exp(a_{ik}) } 
\end{align}
and
\begin{align}\label{eq:post_softmax_attention_B}
    \alpha_{ij}^{(b)} &= \frac{w_i^{(b)} exp(a_{ij})}{\sum_{k=1}^{n} w_k^{(b)} m_k^{(b)} exp(a_{ik}) }
\end{align}
Using the condition provided by Equation (\ref{eq:condition}), Equation (\ref{eq:post_softmax_attention_B}) can be rewritten as:
\begin{align}
    \alpha_{ij}^{(b)} &= \frac{w_i^{(b)} exp(a_{ij})}{\sum_{k=1}^{n} w_k^{(a)} m_k^{(a)} exp(a_{ik}) } \\
    \sum_{k=1}^{n} w_k^{(a)} m_k^{(a)} exp(a_{ik}) &= \frac{w_i^{(b)} exp(a_{ij})}{  \alpha_{ij}^{(b)} }
\end{align}
Substituting back into Equation (\ref{eq:post_softmax_attention_A}) gives us:
\begin{align}
        \alpha_{ij}^{(a)} &= \alpha_{ij}^{(b)} \frac{w_i^{(a)} exp(a_{ij})}{ w_i^{(b)} exp(a_{ij}) } \\
        \alpha_{ij}^{(a)} &= \alpha_{ij}^{(b)} \frac{w_i^{(a)} }{ w_i^{(b)}  }
\end{align}
The attention weights produced from $A$ can now be expressed in terms of the attention weights produced from $B$. The $j^{th}$ entry in the attention vector for $\vx_i$ in $A$ is:
\begin{align}
    y_{ij}^{(a)} &= \sum_{j=1}^{n} \alpha_{ij}^{(a)} m_j^{(a)} v_{ji} \\
    & = \sum_{j=1}^{n} \alpha_{ij}^{(b)} \frac{w_i^{(a)} }{ w_i^{(b)}  } m_j^{(a)} v_{ji} \\
    & = \sum_{j=1}^{n} \alpha_{ij}^{(b)} \frac{w_i^{(a)} }{ w_i^{(b)}  } \bigg(\frac{w_i^{(b)} m_i^{(b)}}{w_i^{(a)}}\bigg) v_{ji} \\
    & = \sum_{j=1}^{n} \alpha_{ij}^{(b)} m_i^{(b)} v_{ji}\\
    & = y_{ij}^{(b)}
\end{align}
Thus, the resulting embeddings from the attention mechanism are equivalent. While the embeddings themselves are equivalent, the cardinality for each embedding still differs according to each multisets' multiplicity. The pooling described by Equation (\ref{eq:pooling}) fixes this. Let $\vz_i$ be the final embedding for $\vx_i$. The output vector $\vz$ from the pooling for $A$ is the sum:
\begin{align}
    \vz^{(a)} &= \sum_{i=1}^{n} w_i^{(a)} m_i^{(a)} \vz_i
\end{align}
Substituting Equation (\ref{eq:condition}) again we get,
\begin{align}
    \vz^{(a)} &= \sum_{i=1}^{n} w_i^{(b)} m_i^{(b)} \vz_i \\
    & = \vz^{(b)}
\end{align}
Thus, the output of the PST for both $A$ and $B$ are equivalent.
\end{proof}

\section{Details of MatBench and Jarvis-DFT Datasets}\label{sec:dataset_details}

\subsection{Materials Project}

The crystals in the Materials Project \cite{MP1} are acquired through the use of \textit{MatBench}\cite{matbench}. MatBench provides standardized training and test sets to ensure that property prediction models can be compared on equal footing. Each training and testing set of the five-fold cross-validation consists of $80\%$ and $20\%$ of the total data respectively. Validation is taken from the training set. In our application we use just $1\%$ of the training data for validation. Due to the relatively simple architecture of the PST, overfitting a target property is more difficult and the typical procedure of selecting the model that performed the most accurately on the validation set is unnecessary. Instead, the PST is trained until the validation error converges.

In Table \ref{tab:matbench_dataset} we list the details of the MatBench dataset for the properties we tested. The information concerning the source of the data and how each of the properties for the crystals therein were calculated as follows.

\begin{itemize}
    \item Formation Energy - Crystals from this dataset are sourced directly from the Materials Project \cite{MP1}. The property is measured in $eV/atom$ and is calculated using DFT-GGA \cite{perdew1996generalized_PBE}.
    \item Band Gap Energy - Crystals in this dataset are taken from the Materials Project\cite{MP1}. The property is measured in $eV$ and is calculated using DFT-GGA \cite{perdew1996generalized_PBE}.
    \item Bulk/Shear Modulus - Crystals are sourced from the Materials Project. The value of the property is measured in $log_{10}(GPa)$. The elastic modulus tensor is derived from the full stress tensor which is calculated using the projector augmented wave method with DFT-GGA functionals \cite{blochl1994projector_paw,perdew1996generalized_PBE}.
    \item Perovskite Formation Energy - Crystals are taken from the work produced by Castelli et al. \cite{castelli2012new_perovskites}. Property values are measured in $eV/unit$ $cell$ and are calculated with DFT-GGA functionals using the RPBE approximation \cite{hammer1999improved_rpbe_approx}.
    \item Phonon Peak - Crystals are taken from the Materials Project. Property values are measured in $1/cm$ and are calculated using DFPT-GGA with PBEsol \cite{petretto2018high_phonon,perdew2008restoring_pbesol}.
    \item Refractive Index - The crystals are taken from the Materials Project. Property values are derived directly from the dielectric tensor \cite{petousis2017high_dielectric}. Dielectric tensors are calculated using Density Functional Perturbation Theory using VASP with PBE+U exchange-correlation functional\cite{perdew1996generalized_PBE,dudarev1998electron_plus_U} and Projector Augmented Wave psuedopotentials \cite{blochl1994projector_paw}.
    \item Exfoliation Energy - Crystals are part of the Jarvis-DFT \cite{choudhary2020joint_jarvis_og} \textit{2D} materials dataset, but are originally sourced from the Materials Project and ICSD \cite{hellenbrandt2004inorganic_icsd}. Property values are measured in $meV/atom$ and are calculated using DFT with optB88 functional \cite{choudhary2017high_jdft2d}. Note, that these are different from the crystals mentioned below as part of the Jarvis-DFT dataset for exfoliation energy, which are part of the \textit{3D} materials dataset. 
\end{itemize}

\begin{table}[h]
    \centering
        \caption{Details of the MatBench dataset by property, including the number of samples in each dataset and the mean-absolute-deviation (MAD).}
    \begin{tabular}{|c|c|c|c|}
    \hline
        Property & Units & Samples & MAD \\
        \hline
        Formation Energy & $eV/atom$ & 132,752 & 1.006 \\
        Band Gap Energy & $eV$ & 106,113 & 1.327 \\
        Shear Modulus & $log_{10}(GPa)$ & 10,987 & 0.293  \\
        Bulk Modulus & $log_{10}(GPa)$ & 10,987 & 0.290  \\
        Refractive Index & n/a & 4,764 & 0.809 \\
        Phonon Peak & $1/cm$ & 1,265 & 323.787  \\
        Exfoliation Energy & $meV/atom$ & 636 & 67.202 \\
        Perovskites FE & $eV/cell$ & 18,928 & 0.566 \\
    \hline
    \end{tabular}
    \label{tab:matbench_dataset}
\end{table}

\subsection{Jarvis-DFT}

Much of the Jarvis-DFT database consists of crystals taken from the Materials Project \cite{MP1} and the ICSD \cite{hellenbrandt2004inorganic_icsd}. After the initial structures are obtained, they are geometrically optimized using DFT, and then their properties are calculated. Details for each property are as follows:

\begin{itemize}
    \item Formation Energy - The crystals in this dataset are taken from the Materials Project. Their values are measured in $eV/atoms$ and are calculated using OptB88vdW functional \cite{klimevs2009chemical_opt_bg}.
    \item Band Gap Energy (OPT) - The crystals in this dataset are taken from the Materials Project. The energy values are measured in $eV$ and are calculated using the OptB88vdW functional \cite{klimevs2009chemical_opt_bg,choudhary2018computational_opt_bandgap_og}.
    \item Total Energy - The crystals in this dataset are taken from the Materials Project and ICSD. Their values are measured in $eV/atom$ and are calculated using OptB88vdW functional \cite{klimevs2009chemical_opt_bg}.
    \item Bulk/Shear Modulus - The crystals are sourced from the Materials Project. Property values are measured in $GPa$ and are calculated using the projector augmented wave method using the OptB88vdW functional \cite{choudhary2018elastic_jarvis_bulk_shear_mod}.
    \item Band Gap Energy (MBJ) - The crystals are taken from the Materials Project. The property values are measured in $eV$ and are computed using the modified Becke-Johnson potential\cite{tran2009accurate_mbj_functionals, choudhary2018computational_opt_bandgap_og}.
    \item Spillage - Crystals are taken from the Materials Project. The structures are relaxed using the OptB88vdW functional but Spin-Orbit spillage is calculated using the PBE functional \cite{choudhary2019high_spillage}.
    \item SLME - Crystals are taken from the Materials Project. Spectroscopic Limited Maximum Efficiency (SLME) is calculated from the dielectric function and band gap (which are computed using DFT)\cite{choudhary2020joint_jarvis_og}.
    \item Maximum Piezoelectric Stress Coefficient $(e_{ij})$ - Crystals taken from the Materials Project. Value is measured in $Cm^{-2}$ and is computed by taking the maximum entry from the Piezoelectric stress tensor. 
    \item Maximum Piezoelectric Strain Coefficient $(d_{ij})$ - Crystals taken from the Materials Project. Property value is measured in $CN^{-1}$ and is computed by taking the maximum entry from the Piezoelectric strain tensor. 
    \item Exfoliation Energy - Property values are measured in $meV/atom$ and are calculated using DFT with optB88 functional \cite{choudhary2020joint_jarvis_og}. These structures are different from the crystals mentioned as part of the MatBench dataset for exfoliation energy, which are part of the \textit{2D} materials dataset. 
\end{itemize}

\section{Applications of the Pointwise Distance Distribution}\label{sec:how_to_use_pdd}

In definition \ref{def:pdd} we provided a formal definition for the PDD. Here, we will begin by providing an example of its construction. 

\begin{figure*}[h!]
    \centering
        \includegraphics[width=.4\textwidth]{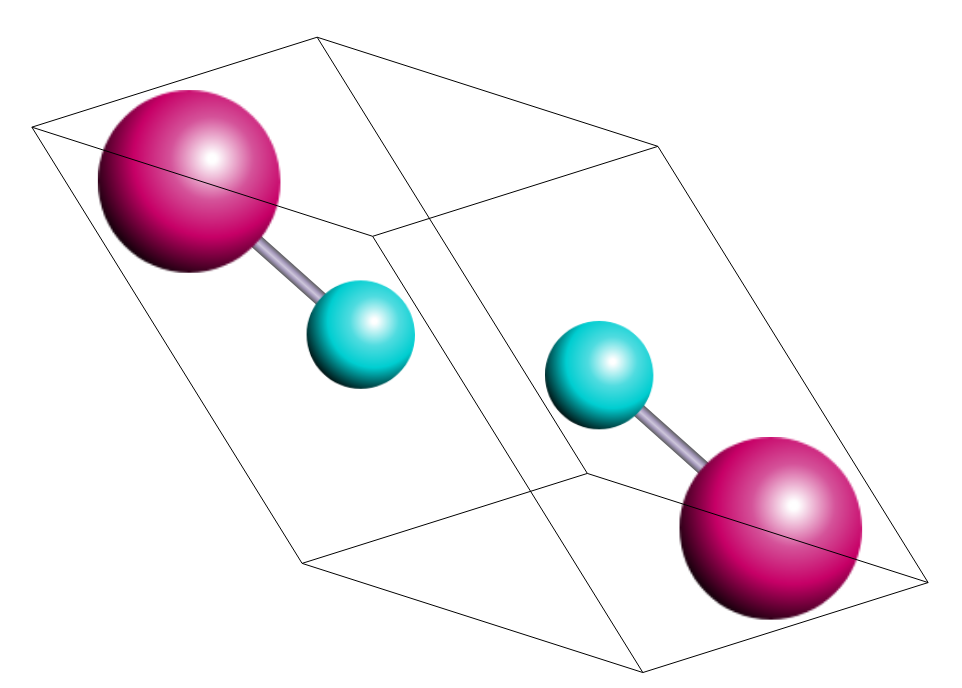}
        \caption{The unit cell of Lutetium-Silicon; Silicon is colored in teal and Lutetium in magenta.}
        \label{fig:LuSi}
\end{figure*}

Consider Lutetium-Silicon which is shown in Figure \ref{fig:LuSi}. The unit cell pictured contains a total of four atoms, 2 of which are Lutetium and 2 of which are Silicon. If we use  $k=2$ nearest neighbors, the PDD matrix of this periodic set $S$ before rows are grouped is
$$ PDD(S; k) = \begin{pmatrix}
0.25 & 2.481 & 2.481 \\
0.25 & 2.481 & 2.481 \\
0.25 & 2.881 & 2.881 \\
0.25 & 2.881 & 2.881
\end{pmatrix} $$
where the first column contains the weights for each row (atom). The second is the distance to the nearest neighbor and the third, is the distance to the second nearest neighbor. The first two rows are identical, as are the final two. Because of this, each of the two groups of rows can be grouped into a single row like so,
$$ PDD(S; k) = \begin{pmatrix}
0.5 & 2.481 & 2.481 \\
0.5 & 2.881 & 2.881
\end{pmatrix} $$
The rows are already lexicographically ordered so this is the finalized PDD. If a different unit cell is selected, this collapsing of matrix rows will continue to yield the same result as the proportion of each atom will not change.

The PDDs of two periodic crystals can be compared using the Earth Mover's Distance \cite{rubner2000earth}. The weights of each PDD can be considered the distribution we are comparing in the minimum flow problem. We can visualize a set of crystals by projecting these distances onto a two or three-dimensional plane.

Subfigure \ref{fig:scatter2d} shows the projection of the pairwise distances of the PDDs of one hundred samples from each of the three crystals in this experiment using Multi-Dimensional Scaling (MDS). This technique attempts to map a set of pairwise distances to specific $n$-dimensional space by minimizing the difference between the actual pairwise distances and the distances between the points in the projected space. Even with this error, we are able to establish three distinct clusters in two-dimensional space according to their respective crystal type. 

\begin{figure*}[h!]
    \begin{subfigure}[b]{.49\columnwidth}
        \includegraphics[width=\textwidth]{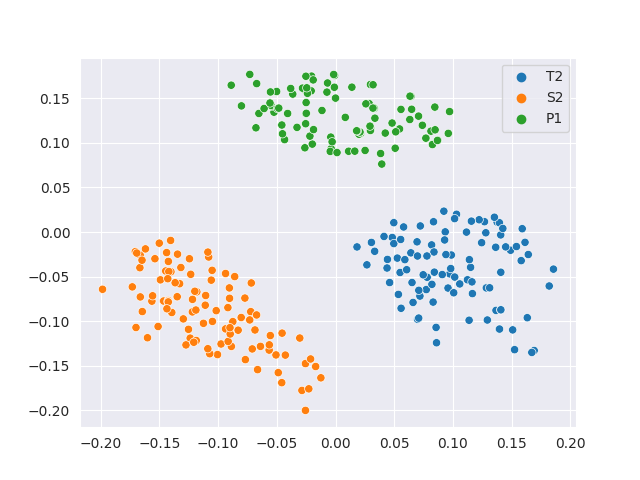}
        \caption{MDS of sampled data by crystal type}
        \label{fig:scatter2d}
    \end{subfigure}
    \hfill
    \begin{subfigure}[b]{.49\columnwidth}
        \includegraphics[width=\textwidth]{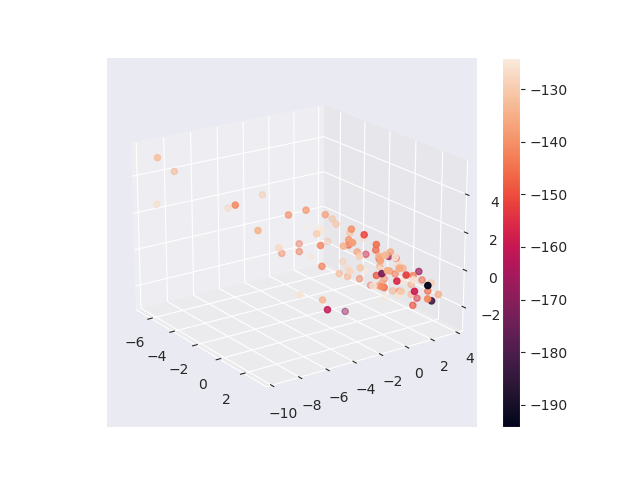}
        \caption{MDS of T2 crystals}
        \label{fig:T2scatter3d}
    \end{subfigure}

    \caption{Multi-dimensional scaling projection \cite{Borg2005_mds} on to $\mathbb{R}^2$ and $\mathbb{R}^3$ for the pairwise distances of crystals between each other. Subfigure (a) projects three types of crystals using distances created by using the Earth Mover's Distance between PDDs for $k=10$. Subfigure (b) is created using the MDS projection of pairwise distances between PDDs at $k=100$ for one hundred random samples from the T2 crystals colored by lattice energy measured in kJ/mol.}
\end{figure*}

Subfigure \ref{fig:T2scatter3d} shows the projection of the pairwise distances between the crystals' PDD onto three-dimensional space for the T2 dataset. The colors of the points in each plot signify the lattice energy for the crystal. The distinction here is not as pronounced; this is to be expected as the sampled crystals are much more similar in their structure and thus the distances between the PDDs are smaller. Nonetheless, there is a discernible trend, and crystals that congregate near each other do share similar lattice energies.

\section{Prediction of Lattice Energy}\label{sec:le_results}

Here, the dataset considered contains simulated molecular crystals created by Pulido et al. \cite{Pulido2017-ff} during the crystal structure prediction using quasi-random sampling \cite{case2016convergence_quasi_random}. This data is subsetted by the underlying molecule, e.g. T2. During structure prediction, crystals are generated while traversing the potential energy surface and their lattice energy is calculated using \textit{DMACRYS} \cite{price2010modelling_lattice_energy_calculations} to determine their stability. 

Lattice energy is defined as the energy released during the crystallization of the constituent molecules into a lattice. Prediction of lattice energy is done in three different scenarios. In the first, the model will be applied to a single set of molecular crystals with the same composition using $80\%/10\%/10\%$ training, validation, and testing splits. Next, the model is applied to multiple sets of crystals each with different underlying molecules using the same splits. The final experiment consists of the application of the model to a set of crystals with an underlying molecule it has not seen in training. The seen data is split $90\%/10\%$ for training and validation. To make predictions the PST described in section \ref{sec:PST} is used with only the PDD as input and no knowledge of the composition.

Invariants are usually used to discern crystals by measuring differences between their structure. Here, the goal is to demonstrate the effectiveness of using an invariant as a representation for a machine learning algorithm. Even when compositional information is not present, the PDD can distinguish crystals with the EMD from the changes in the pairwise distances that occur when the species of atoms are changed. Whether the same distinction can be made in the context of a learning algorithm has previously not been shown.

We make a comparison to another invariant that has been used to predict lattice energy. Average-minimum-distance \cite{widdowson2022average} was used as input for a Gaussian regression model \cite{amd_le}. In Table \ref{tab:exp1} we list the performance on the test set of this AMD model to allow comparison between invariants. 

Our model reduces the mean-absolute-error (MAE) rate by $21\%$ compared to the Gaussian Process Regression technique which utilizes AMD \cite{amd_le}. The mean-absolute percentage error (MAPE) is also reduced by $1.63\%$. While we use $k=60$ nearest neighbors for constructing the PDD, the AMD model uses $k=500$ to achieve its best results. Despite the use of this additional information, the model using the PDD still performs more accurately.

\begin{table}[h]
\small
\caption{\label{tab:exp1}%
\small Results of lattice energy prediction using the PDD with the PST and AMD with Gaussian regression on three different tasks using $80\%/10\%/10\%$ training, validation, and testing splits. \textbf{(a)} uses only experimentally generated structures of the T2 molecular crystal based on triptycene. \textbf{(b)} adds two additional sets of molecular crystals, P1 (based on pentiptycene) and S2 (based on spiro-biphenyl). \textbf{(c)} uses P1, P1M (methylated analogue of P1) and P2 (benzimidazolone series pentiptycene) in the training and validation set using a $90\%/10\%$ split. The test set consists only of P2M, the methylated analogue of P2.} 
\centering
\begin{tabular}{|c|c|c|c|c|c|c|c| }
\hline
&
Train&
Train Size&
Test&
Test Size&
Invariant&
MAE (kj/mol) $\downarrow$&
MAPE $\downarrow$ \\
\hline
\multirow{2}{*}{\textbf{(a)}} & \multirow{2}{*}{T2} & \multirow{2}{*}{4,630} & \multirow{2}{*}{T2} & \multirow{2}{*}{578} & AMD & \hfil 4.79 & \hfil 4.31\%\\
& & & & & PDD & \hfil 3.76 & \hfil 2.68\%\\
\hline 
\multirow{2}{*}{\textbf{(b)}} & \multirow{2}{*}{T2, P1, S2} & \multirow{2}{*}{14,547} & \multirow{2}{*}{T2, P1, S2} & \multirow{2}{*}{1,819} & AMD & \hfil 4.68 & \hfil 2.83\%\\
& & & & & PDD & \hfil 4.11 & \hfil 2.52\%\\
\hline
\multirow{2}{*}{\textbf{(c)}} & \multirow{2}{*}{P1,P1M,P2} & \multirow{2}{*}{22,995} & \multirow{2}{*}{P2M} & \multirow{2}{*}{7,352} & AMD & \hfil 12.99 & \hfil 6.89\% \\
& & & & & PDD & \hfil 7.24 & \hfil 3.89\%\\
\hline
\end{tabular}
\end{table}

In the second row of Table \ref{tab:exp1} we list the results of the second experiment. The previous task is extended to a dataset of crystals that contains different underlying molecules. These molecules have different compositions. This compositional information is not contained within the PDD (and thus, not in our input). The model will have to discern crystals solely by their structure. While the overall MAE has increased slightly, the percentage error has decreased. The domain on which the lattice energies lie is different for each type of crystal. Using the PDD alone is enough for the algorithm to distinguish the crystal types and predict lattice energy accordingly.

The final experiment uses the data from the P1, P1M, and P2 crystals in the training and validation data. The test set consists of the P2M crystal, which is not part of the either training or validation set. This experiment is the closest to real-world conditions in which new crystals often arise and finding the stable forms is crucial, but information on their lattice energy is unavailable.

When lattice energy is calculated using ab initio calculations, the range of the energies varies from crystal to crystal. When introducing a new type of crystal for our algorithm to make predictions on, this can become a problem as extrapolation to unvisited parts of the lattice energy range can result in high error rates. Fortunately, lattice energies between different types of crystals, are not usually meant to be compared. Instead, they are generated for potential polymorphs of a crystal in an effort to find those with the highest stability for synthesis. We can make use of this fact when applying our model to novel crystals. It is not necessary to predict the correct range of lattice energies; instead, the model needs to be able to predict the lattice energies of the various structures such that their ordering according to their lattice energy is correct. This task could feasibly be turned into a \textit{learning-to-rank} problem \cite{learning_to_rank}, but as a regression task, it allows for a more general approach since the predicted lattice energies can be ordered after the fact.

Each dataset has its lattice energies shifted by the mean lattice energy towards zero. By doing this they each maintain their distribution but now overlap around the origin. The model is trained and validated based on this shifted data. The MAE of the predictions on the test set after they have been shifted back is $7.24$ kJ/mol and the MAPE is $3.89\%$.

While the MAE and MAPE are higher than in previous experiments, the improvement over AMD is more significant. The majority of errors come from underestimating the lattice energy. The datasets tend to grow sparser in these areas where lattice energy is lower as this is where the most stable structures tend to lie. Having a false positive (predicting a higher energy structure to be lower) increases the number of potentially stable structures. False negatives are more impactful as they may result in a structure not being considered entirely due to its seemingly low stability.

The histogram in Figure \ref{fig:hist_gt} shows the lattice energies (in kJ/mol) of the three crystals within the dataset. Figure \ref{fig:exp2_gt_vs_pred} shows the comparison of the predictions of the model against the true property values in the dataset, referred to as ground truth values. 

\begin{figure*}[h!]
    \begin{subfigure}[b]{.48\columnwidth}
        \includegraphics[width=\textwidth]{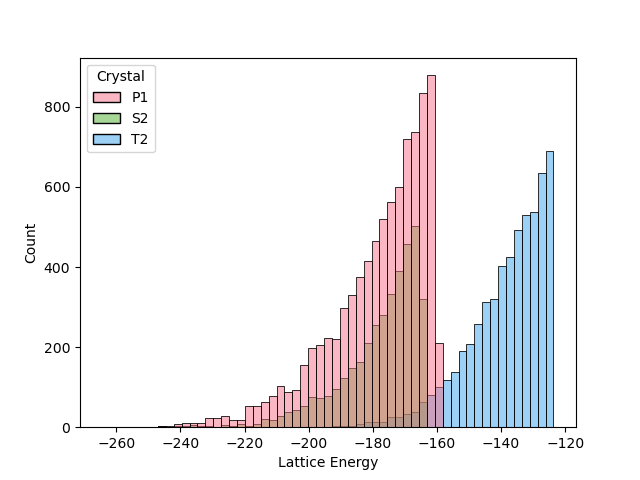}
        \caption{Distribution of ground truth lattice energies.}
        \label{fig:hist_gt}
    \end{subfigure}
    \hfill
    \begin{subfigure}[b]{.48\columnwidth}
        \includegraphics[width=\textwidth]{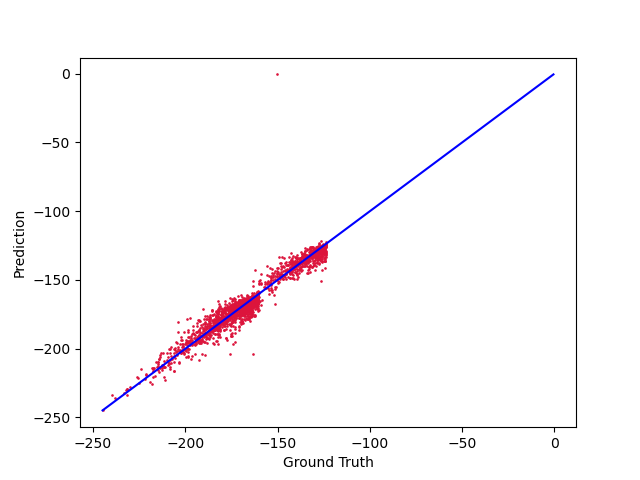}
        \caption{Ground truth lattice energy vs. predictions}
        \label{fig:exp2_gt_vs_pred}
    \end{subfigure}
    \caption{(a) The distribution of the lattice energies of the T2, S2 and P1 crystals. (b) The predictions of T2, S2, and P1 compared to the ground truth lattice energies in kJ/mol. }
\end{figure*}

Predictions that have lower errors will have their point placed closer to the line colored in blue. The bulk of the points share their error both below and above the true lattice energy as we would expect in a model without bias. There are a few outliers, in particular, a single crystal from the T2 dataset has a predicted lattice energy of just $-0.41$. Prevention of such errors can be mitigated by using a different loss function than MAE. In particular, mean-squared error (MSE) and Huber loss can hedge against outlying errors by increasing their contribution to the loss function. We choose to not present the results using these loss functions as MAE still provides better results for MAE and MAPE.

\begin{figure*}[h]
    \begin{subfigure}[b]{.48\columnwidth}
        \includegraphics[width=\textwidth]{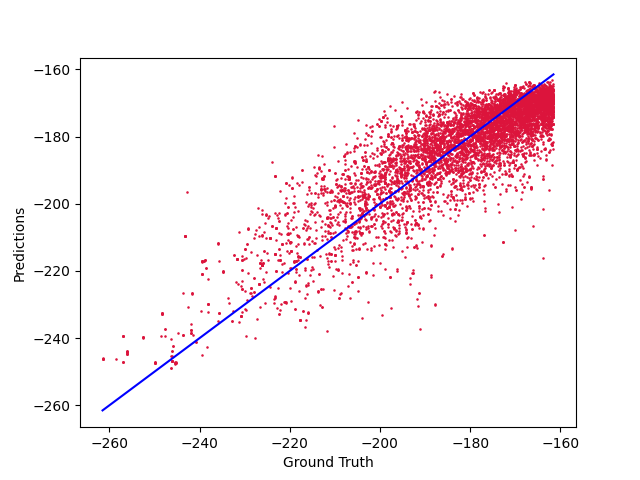}
        \caption{Ground Truth vs. Predictions re-scaled according to mean.}
        \label{fig:exp3_1_scatter}
    \end{subfigure}
    \hfill
    \begin{subfigure}[b]{.48\columnwidth}
        \includegraphics[width=\textwidth]{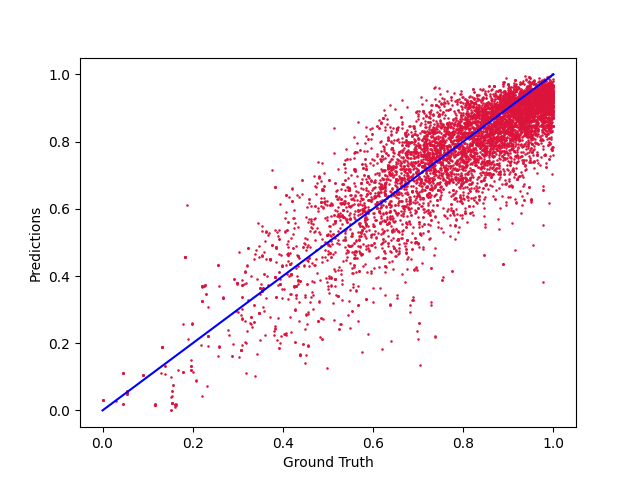}
        \caption{Ground Truth vs. Predictions normalized between zero and one.}
        \label{fig:exp3_2_scatter}
    \end{subfigure}

    \caption{(a) Comparison of predictions vs. true (ground truth) values of the test set after the predictions are scaled back by the mean of the lattice energies. (b) The comparison of predictions vs. ground truth after both have been normalized between zero and one. }
    \label{fig:four figures}
\end{figure*}

In Figure \ref{fig:exp3_1_scatter} we plot the original predictions on P2M after they have been re-scaled back away from the mean lattice energy. In Figure \ref{fig:exp3_2_scatter}, we re-scale the prediction and the true property values found in the dataset (henceforth, referred to as ground truth values) to between zero and one. By doing this, we can see the ordering of the predictions compared to the true lattice energies. The scatter plot in Figure \ref{fig:exp3_2_scatter} allows us to compare lattice energies relative to other predictions. 

If the error rate produced by the final experiment of section \ref{sec:le_results} is inadequate, we can supplement the training with predictions from classical methods. For a novel crystal, ab initio calculations can be used to generate lattice energies for a small subset of structures. These samples can be integrated into the training set to improve the overall results. In order to make this practical, we can only produce lattice energies for a limited portion of structures; we limit this to $10\%$ (or under) of the total structures.

\begin{table}[h]
\caption{\label{tab:exp3_2}%
Effect of including portions of the P2M data into the training set on prediction accuracy. }
\centering
\begin{tabular}{|c|p{3cm}|p{3cm}|c|}
\hline
Samples&
Percent of P2M Data&
Test MAE (kj/mol) $\downarrow$&
Test MAPE $\downarrow$ \\
\hline
 0 & \hfil 0\% &\hfil 7.24 &\hfil 3.89\%\\
367 &\hfil 5\% &\hfil 5.82 & \hfil 3.17\%\\
735 &\hfil 10\% &\hfil 5.50 &\hfil 3.02\%\\
\hline
\end{tabular}
\end{table}

Table \ref{tab:exp3_2} displays the result of adding this supplemental data to the training set. As expected, the addition of data decreases MAE and MAPE. Even with as few as $367$ samples (or $5\%$ of the total dataset), the reduction in error is significant. This process experiences diminishing returns as the amount of the original dataset is used in the training set. Though the MAE and MAPE continue to decrease, it is questionable whether or not spending time generating the instances using classical methods is worth the additional performance gains.

\section{Additional Experiments}\label{sec:addition_experiments}

\subsection{Identifying Crystals with Outlying Property Values}\label{sec:mp_classification}

The properties that have been targeted with the PST are all regression tasks. These can also be reframed as a classification task wherein the model attempts to identify if a given material has a "high" (or alternatively, a "low") property value. Due to the ambiguity of this term, a threshold at which a property qualifies as having a "high" value needs to be established. In the following experiments, a value exceeding the $90^{th}$ percentile of the test set is considered "high".

MAE is no longer a suitable metric for measuring the model's performance. Instead, four other metrics are used: accuracy, recall, and precision. We define the term \textit{true positive} (TP) as the correct classification of materials with a "high" property value. \textit{False positive} (FP) is the incorrect classification of a sample with a "high" property value. \textit{True negative} (TN) refers to the correct classification of a crystal that does not have a "high" property value. Finally, \textit{False Negative} refers to the incorrect classification of material that does not have a "high" property value. Accuracy is determined by the expression $(TP + TN) / (TP + TN + FP + FN)$. Precision is found by evaluating the expression $TP / (TP + FP)$ and recall is found using $TP / (TP + FN)$. The results of this classification task for all properties provided by MatBench are listed in Table \ref{tab:matbench_classification_90}. The scores listed are the average across all five folds of each dataset.

\begin{table}[h]
    \centering
        \caption{Results of the classification of materials with property values in the $90^{th}$ percentile on the MatBench datasets. Threshold refers to the value at which a crystal is determined to be a positive sample if its own property value is higher. The abbreviations in the Table are as follows: (TP) True Positive, (FP) False Positive, (FN) False Negative, (TN) True Negative, and (Acc.) Accuracy. }
    \begin{tabular}{|c|c|c|c|c|c|c|c|c|}
    \hline
        Property & Threshold & TP & FP & FN & TN & Acc. (\%) & Recall (\%) & Precision (\%) \\
        \hline
        Formation Energy & -0.015 $eV/atom$ & 12,309 & 966 & 633 & 118,844 & 98.8 & 95.1 & 92.7 \\
        Band Gap Energy & 3.69 $eV$ & 9,489 & 1,126 & 1,124 & 94,374 & 97.9 & 89.4 & 89.4 \\
        Shear Modulus & 1.99 $log_{10}(GPa)$ & 877 & 204 & 179 & 9,772 & 96.5 & 83.1 & 81.1  \\
        Bulk Modulus & 2.29 $log_{10}(GPa)$ & 955 & 132 & 135 & 9,765 & 97.6 & 87.6 & 87.9  \\
        Refractive Index & 3.392 & 386 & 94 & 78 & 4,206 & 96.4 & 83.2 & 80.4 \\
        Phonon Peak & 1,128 $1/cm$ & 123 & 7 & 8 & 1,127 & 98.8 & 93.4 & 94.6  \\
        Exfoliation Energy & 204.4 $meV/atom$ & 45 & 20 & 18 & 553 & 94.0 & 71.4 & 69.2 \\
        Perovskites FE & 2.436 $eV/cell$ & 1,820 & 50 & 78 & 16,890 & 99.3 & 95.9 & 97.3 \\
    \hline
    \end{tabular}
    \label{tab:matbench_classification_90}
\end{table}

The accuracy for all properties is relatively high, with a lower bound of $94\%$ occurring for the property with the fewest number of samples: exfoliation energy. The trend is similar for recall and precision, with properties containing a smaller sample size having lower scores. Across all properties, the relationship between precision and recall remains balanced with no single property containing an exceedingly high rate of false negatives or false positives. 

\begin{table}[h]
    \centering
        \caption{Results of the classification of materials with property values in the $95^{th}$ percentile on the MatBench datasets. Threshold refers to the value at which a crystal is determined to be a positive sample if its own property value is higher. The abbreviations in the Table are as follows: (TP) True Positive, (FP) False Positive, (FN) False Negative, (TN) True Negative, and (Acc.) Accuracy. }
    \begin{tabular}{|c|c|c|c|c|c|c|c|c|}
    \hline
        Property & Threshold & TP & FP & FN & TN & Acc. (\%) & Recall (\%) & Precision (\%) \\
        \hline
        Formation Energy & 0.178 $eV/atom$ & 6,015 & 625 & 380 & 125,732 & 99.2 & 94.1 & 90.6 \\
        Band Gap Energy & 4.55 $eV$ & 4,662 & 648 & 604 & 100,199 & 98.9 & 88.6 & 87.8 \\
        Shear Modulus & 2.08 $log_{10}(GPa)$ & 408 & 128 & 98 & 10,353 & 98.0 & 81.0 & 76.1  \\
        Bulk Modulus & 2.35 $log_{10}(GPa)$ & 465 & 73 & 83 & 10,366 & 98.6 & 84.9 & 86.4  \\
        Refractive Index & 4.353 & 168 & 72 & 65 & 4,459 & 97.1 & 72.1 & 70.0 \\
        Phonon Peak & 1,428 $1/cm$ & 62 & 3 & 8 & 1,192 & 99.1 & 88.6 & 95.4  \\
        Exfoliation Energy & 318.0 $meV/atom$ & 17 & 18 & 9 & 592 & 95.8 & 65.4 & 48.6 \\
        Perovskites FE & 2.891 $eV/cell$ & 887 & 40 & 38 & 17,963 & 99.6 & 95.9 & 95.7 \\
    \hline
    \end{tabular}
    \label{tab:matbench_classification_95}
\end{table}

Increasing the threshold at which crystals are classified as positive decreases both recall and precision, but accuracies increase. This can be seen in Table \ref{tab:matbench_classification_95} where the threshold is increased to the $95^{th}$ percentile. At this level, most properties still perform well with only small reductions in recall and precision. Exfoliation energy is an exception to this. This is the byproduct of having a comparatively small dataset.

If materials with property values above the $99^{th}$ percentile are considered positive samples the precision and recall of refractive index drops to $24\%$ and $14\%$ percent respectively. Notably, at this threshold, there are only $7$ such crystals which qualify as positive. Datasets with a large number of crystals do not experience as drastic a decrease in efficacy. Band Gap energy, for example, experiences a reduction to $87.1\%$ and $80.4\%$ for precision and recall, respectively. 

\subsection{Effect of $k$-nearest neighbors}

PDD encoding can be said to be parameterized by two values, the collapse tolerance and the number of $k$-nearest neighbors. The integer $k$ determines the dimensionality of initial PDD encoding embedding. As $k$ increases, it retains all information from the previous (smaller) values of $k$. The initial nearest neighbor distances are the most important and the embedding has diminishing returns after this. For any value of $k > 1$, the PDD is invariant. Thus, if the PDD is different, the crystals are guaranteed to be structurally different. In order for the PDD to be distinct such that if any two crystals are different, their PDDs are different, we need the property of generic completeness. This property is given provided the lattice $L$ and sufficiently large $k$. An upper bound on this $k$ is when all distances in the last column of the PDD are larger than twice the covering radius of the lattice $L$ of the periodic set. We would expect the performance of the model to increase up until this point. 

The upper bound on $k$ can be exceedingly large so we implement a heuristic to find a lower bound that is more computationally efficient. As $k$ increases, the number of rows collapsed in the PDD will either stay the same or decrease. The lower bound on $k$ can be considered an integer large enough that the groups established at the upper bound of $k$ are the same as this lower bound. Each crystal could have a different $k$ for which this requirement is met. Our encoding method prohibits the use of a dynamic $k$ value, therefore we need a consistent value for $k$ that can be applied to all crystals in the dataset. The results in Table \ref{tab:mpresults1} use $k=15$. At this value of $k$ the previously mentioned lower bound is satisfied for $99.1\%$ within the formation energy dataset. Increases in $k$ past this point cause marginal improvements to this coverage that were deemed insufficient when the increased computational cost is considered.

\begin{table}[h]
    \centering
  \begin{tabular}{|p{4cm}|p{2.1cm}|p{1.9cm}|p{1.7cm}|p{1.6cm}|}
    \hline
    \multirow{2}{*}{\hfil Property (units)} &
      \multicolumn{4}{c|}{MAE by $k$-Nearest Neighbors PDD Encoding $\downarrow$ } \\
      \cline{2-5}
    & \hfil $5$ & \hfil $10$ & \hfil $15$ & \hfil $20$    \\
    \hline
    Band Gap \hspace{0.1cm} \small{$eV$} & \hfil 0.261 & \hfil 0.232 & \hfil \textbf{0.212} & \hfil \underline{0.214}  \\ 
    Formation \hspace{0.1cm} \small{$eV/atom$} & \hfil 0.039 & \hfil 0.034 & \hfil \underline{0.032} & \hfil \textbf{0.032}  \\ 
    Shear Modulus \hspace{0.1cm} \small{$log_{10}(GPa)$} & \hfil 0.087 & \hfil 0.077 & \hfil \underline{0.075} &  \hfil \textbf{0.075} \\ 
    Bulk Modulus \hspace{0.1cm} \small{$log_{10}(GPa)$} & \hfil 0.064 & \hfil 0.059 & \hfil \underline{0.055} &  \hfil \textbf{0.056} \\ 
    Refractive Index  & \hfil 0.318 & \hfil 0.291 & \hfil \underline{0.292} & \hfil \textbf{0.283} \\ 
    Phonon Peak \hspace{0.1cm} \small{$1 / cm$} & \hfil \textbf{26.38} & \hfil 27.89 & \hfil \underline{27.74} &  \hfil 29.21 \\ 
    Exfoliation \hspace{0.1cm} \small{$meV/atom$} & \hfil 38.77 & \hfil 33.63 & \hfil \textbf{31.55} & \hfil \underline{31.70}  \\ 
    Perovskites FE \hspace{0.1cm} \small{$eV/cell$} & \hfil 0.031 & \hfil 0.031 & \hfil \underline{0.030} & \hfil \textbf{0.030} \\ 
    \hline
  \end{tabular}
    \caption{ Prediction MAE on the Materials Project crystals for various $k$-nearest neighbors PDD Encoding at a collapse tolerance of exactly zero. Errors in bold indicate the value of $k$ with the best performance and underlined errors indicate the second-best performance (lower is better $\downarrow$). }
    \label{tab:mpbyk}
\end{table}

The error rates listed in Table \ref{tab:mpbyk} vary the value of $k$ and report the resulting mean MAE across the five folds. As expected, the lower values of $k$ generally result in higher MAE. Increasing $k$ eventually causes the error rates to stop decreasing. This is also in line with what would be expected as the PDD has enough information to distinguish itself and additional distances are unnecessary. This is not the case only with phonon peak, however, the differences in MAE are relatively small when the deviation of the errors across the folds is considered.

\subsection{Effect of collapse tolerance}\label{sec:mp_results_by_tol}

The collapse tolerance dictates which rows of the PDD will be collapsed. As this parameter increases, the size of the grouped rows will increase. Once rows are grouped, their distances are averaged in the row which represents the group. The change in this averaged row is proportional to the size of the collapse tolerance. In $PDD(S;k)$, as the collapse tolerance approaches infinity, the PDD will decrease in the number of rows until it consists of just a single row with a weight equal to one. In $PDD(S;k)$, the same increase in collapse tolerance will result in a number of rows within the PDD equal to the number of unique elements within the crystal. In both cases, a collapse tolerance that is large enough will result in information loss, eventually increasing errors in predictions.  


\begin{table}[h]
    \centering
  \begin{tabular}{|p{4cm}|p{2.1cm}|p{1.9cm}|p{1.7cm}|p{1.6cm}|}
    \hline
    \multirow{2}{*}{\hfil Property (units)} &
      \multicolumn{4}{c|}{MAE by Collapse Tolerance in PDD Encoding $\downarrow$ } \\
      \cline{2-5}
    & \hfil $1.0$ & \hfil $10^{-2}$ & \hfil $10^{-4}$ & \hfil $0.0$    \\
    \hline
    Band Gap \hspace{0.1cm} \small{$eV$}               & \hfil 0.241 & \hfil 0.222 & \hfil \underline{0.210} & \hfil \textbf{0.212}  \\ 
    Formation \hspace{0.1cm} \small{$eV/atom$}         & \hfil 0.037 & \hfil 0.033 & \hfil \underline{0.032} & \hfil \textbf{0.032}  \\ 
    Shear Modulus \hspace{0.1cm} \small{$log_{10}(GPa)$} & \hfil \textbf{0.074} & \hfil 0.074 & \hfil \underline{0.074} &  \hfil 0.075 \\ 
    Bulk Modulus \hspace{0.1cm} \small{$log_{10}(GPa)$}  & \hfil 0.056 & \hfil 0.056 & \hfil \textbf{0.056} &  \hfil \underline{0.055} \\ 
    Refractive Index                                     & \hfil \textbf{0.283} & \hfil 0.292 & \hfil \underline{0.290} & \hfil 0.292 \\ 
    Phonon Peak \hspace{0.1cm} \small{$1 / cm$}          & \hfil \underline{28.41} & \hfil 28.50 & \hfil 29.40 &  \hfil \textbf{27.74} \\ 
    Exfoliation \hspace{0.1cm} \small{$meV/atom$}        & \hfil 32.19 & \hfil \textbf{31.13} & \hfil \underline{31.15} & \hfil 31.55 \\
    Perovskites FE \hspace{0.1cm} \small{$eV/cell$}       & \hfil 0.030 & \hfil 0.030 & \hfil 0.030 & \hfil 0.030 \\ 
    \hline
  \end{tabular}
    \caption{ Prediction MAE on the Materials Project crystals using PDD Encoding at various collapse tolerances with $k=15$. Errors in bold indicate the collapse tolerance with the best performance and underlined errors indicate the second-best performance (lower is better $\downarrow$). }\label{tab:mpbytol}
    
\end{table}

The collapse tolerance is varied and then applied to the Materials Project crystals. The results of this experiment are listed in Table \ref{tab:mpbytol}.

By increasing the collapse tolerance and reducing the number of rows within the PDD, we can increase the speed of computations.  Thus, it is important to choose a tolerance that is maximal, while not sacrificing accuracy. The impact of the collapse tolerance on the size of representation is listed in Table \ref{tab:mpsizebytol}. These values are calculated by dividing the number of rows in the PDD by the number of atoms in the unit cell. The number of atoms in the unit cell is used to determine the number of vertices in the crystal graph \cite{CGCNN}. In this way, the size of our representation can be compared to that of popular graph-based models. Data for crystals typically comes in the form of Crystallographic Information Files (CIF). These files also indicate the amount of potential measurement error for the atomic positions. This is used as a guide and a collapse tolerance of $10^{-4}$ is used in the experiments for the results in Table \ref{tab:mpresults1}. Sometimes, however, a higher collapse tolerance can act as a regularization technique that is useful on smaller datasets. This effect is seen in the results for refractive index and exfoliation energy, but there is still a balance to be struck. In larger datasets, the performance regression is more noticeable. A collapse tolerance of one is far higher than what would be necessary to cover measurement error in atomic coordinates and would not be advised, even with the potential efficiency gains. Overall, the collapse tolerance does not have a very large impact due to the prevention of rows corresponding to different atoms from being collapsed in the PDD. 

\begin{table}[h]
    \centering

  \begin{tabular}{|p{2.3cm}|c|c|c|c|c|c|c|c|c|}
    \hline
    \multirow{2}{*}{\hfil Property} & \multirow{2}{*}{Mean $|M|$} & \multicolumn{4}{c|}{ Size of Input } & \multicolumn{4}{c|}{ Percentage of $|M|$ }   \\
      \cline{3-10}
    & & 0.0 & $10^{-4}$ & $10^{-2}$ & 1.0 & 0.0 & $10^{-4}$ & $10^{-2}$ & 1.0 \\
    \hline
    Phonon Peak & 7.5 & 7.3 & 3.6 & 3.5 & 3.4 & 96.8\% & 55.7\% & 54.5\% & 53.8\%    \\
    Ref. Index & 16.9 & 16.5 & 6.7 & 6.2 & 5.8 & 97.7\% & 49.1\% & 46.3\% & 44.1\%   \\
    Bulk Modulus & 8.6 & 8.3 & 4.1 & 3.9 & 3.7 & 96.7\% & 62.1\% & 60.4\% & 59.4\%   \\
    Shear Modulus & 8.6 & 8.3 & 4.1 & 3.9 & 3.7 & 96.7\% & 62.1\% & 60.4\% & 59.4\%    \\
    Band Gap & 30.0 & 29.2 & 13.1 & 12.0 & 10.1 & 97.0\% & 52.0\% & 48.5\% & 44.2\%    \\
    Formation & 29.1 & 28.4 & 12.7 & 11.6 & 9.9 & 96.9\% & 53.1\% & 49.7\% & 45.7\%    \\
    Exfoliation & 7.2 & 7.1 & 3.6 & 3.3 & 3.2 & 98.8\% & 55.9\% & 51.9\% & 51.5\%   \\
    Perovskites & 5.0 & 4.9 & 4.6 & 4.6 & 4.6 & 99.0\% & 94.8\% & 94.8\% & 94.8\%\\ 
    \hline
  \end{tabular}
 
\caption{\label{tab:mpsizebytol}%
 Size of the input representation for each dataset in the Materials Project at various collapse tolerances at $k=15$ compared to the number of atoms in the unit cell $|M|$. The size of the input refers to the cardinality of the input set determined by the number of rows in the PDD. The percentage of $|M|$ is the input set's cardinality divided by the number of atoms in the unit cell, expressed as a percentage.  }
\end{table}

\section{Implementation Details}\label{sec:implemention_details}

The Periodic Set Transformer is implemented using \textit{PyTorch} \cite{paszke2019pytorch}. There is also a version implemented using \textit{Tensorflow} \cite{abadi2016tensorflow}, however, we have found this version to significantly underperform when compared to the PyTorch version. We believe this to be due to how the output of individual attention head output is handled in their respective implementations of Multi-head Attention. 

Data pre-processing is fairly minimal. Each crystal comes in the form of a \texttt{Pymatgen} structure. The structure is converted into a \texttt{PeriodicSet} object. This functionality is provided by the \textit{AMD} package \cite{widdowson2022average}. The PDD of each of the \texttt{PeriodicSet} objects is then calculated with the desired collapse tolerance and $k$ value. Each column is then normalized to between zero and one. This is not necessary for achieving the desired accuracy but it does significantly improve the speed of training by requiring fewer epochs. 

With respect to the results in Table \ref{tab:mpresults1}, training is done on each property with the same hyper-parameters. The hyper-parameters that govern PDD encoding remain the same for all properties: a tolerance of $10^{-4}$ and $k = 15$. Training options including weight decay, epochs, and learning schedule are kept the same as well, except for batch size. Batch size is either 32 or 64 depending on the number of samples: 32 if the number of crystals in the dataset is less than 5000 and 64 if greater.

\end{document}